\documentclass[letterpaper]{article} 
\usepackage{aaai23}  
\usepackage{times}  
\usepackage{helvet}  
\usepackage{courier}  
\usepackage[hyphens]{url}  
\usepackage{graphicx} 
\usepackage{natbib}  
\usepackage{caption} 
\usepackage{algorithm}
\usepackage{algorithmic}

\usepackage{newfloat}
\usepackage{listings}
\DeclareCaptionStyle{ruled}{labelfont=normalfont,labelsep=colon,strut=off} 
\floatstyle{ruled}
\newfloat{listing}{tb}{lst}{}
\floatname{listing}{Listing}
\usepackage{amsmath, amssymb, amsthm}
\usepackage{enumerate}
\usepackage{subfig}
\usepackage{booktabs, multirow}
\newcommand{\argmin}{\operatornamewithlimits{argmin}}
\newcommand{\argmax}{\operatornamewithlimits{argmax}}
\newcommand{\diag}{\operatorname{diag}}
\newcommand{\KL}{\operatorname{KL}}
\newcommand{\tr}{\operatorname{tr}}
\newtheorem{theorem}{Theorem}
\newtheorem{lemma}{Lemma}
\newtheorem{remark}{Remark}
\newtheorem{assumption}{Assumption}
\graphicspath{{figures/}}
\usepackage{soul}
\usepackage[colorlinks,citecolor=green,linkcolor=red]{hyperref}

\title{Scalable Bayesian Meta-Learning through Generalized Implicit Gradients}
\author{
    Yilang Zhang, 
    Bingcong Li,
    Shijian Gao,
    Georgios B. Giannakis
}
\affiliations{
    Dept. of ECE, University of Minnesota, Minneapolis, MN, USA

    \{zhan7453,lixx5599,gao00379,georgios\}@umn.edu
}

\begin{document}

\maketitle

\begin{abstract}
Meta-learning owns unique effectiveness and swiftness in tackling emerging tasks with limited data. Its broad applicability is revealed by viewing it as a bi-level optimization problem. The resultant algorithmic viewpoint however, faces scalability issues when the inner-level optimization relies on gradient-based iterations. Implicit differentiation has been considered to alleviate this challenge, but it is restricted to an isotropic Gaussian prior, and only favors \emph{deterministic} meta-learning approaches. This work markedly mitigates the scalability bottleneck by cross-fertilizing the benefits of implicit differentiation to \emph{probabilistic} Bayesian meta-learning. The novel implicit Bayesian meta-learning (iBaML) method not only broadens the scope of learnable priors, but also quantifies the associated uncertainty. Furthermore, the ultimate complexity is well controlled regardless of the inner-level optimization trajectory. Analytical error bounds are established to demonstrate the precision and efficiency of the generalized implicit gradient over the explicit one. Extensive numerical tests are also carried out to empirically validate the performance of the proposed method.
\end{abstract}

\section{Introduction}
Over the past decade, deep learning (DL) has garnered huge attention from theory, 
algorithms, and application viewpoints. The underlying success of DL is mainly 
attributed to the massive datasets, with which large-scale and highly expressive 
models can be trained. On the other hand, the stimulus of DL, namely data, can 
be scarce. Nevertheless, in several real-world tasks, such as object recognition and concept comprehension, humans can perform exceptionally well even with very few data samples. This prompts the natural question: \textit{How can we endow DL with human's unique intelligence?} By doing so, 
DL's data reliance can be alleviated and the subsequent model training can be 
streamlined. Several trials have been emerging in those ``stimulus-lacking'' 
domains, including speech recognition~\cite{app-speech}, medical 
imaging~\cite{app-med}, and robot manipulation~\cite{app-robot}. 

A systematic framework has been explored in recent years to address the aforementioned question, under the terms \textit{learning-to-learn} or \textit{meta-learning}~\cite{lifelong-learning}. In brief, meta-learning extracts task-invariant prior information from a given family of correlated (and thus informative)  tasks. Domain-generic knowledge can therein be acquired as an inductive bias and transferred to new tasks outside the set of given ones~\cite{learning-to-learn, LLAMA}, making it feasible to learn unknown models/tasks even with minimal training samples. One representative example is that of an edge extractor, which can act as a common prior owing to its presence across natural images. Thus, using it can prune degrees of freedom from a number of image classification models. The prior extraction in conventional meta-learning is more of a hand-crafted art; see e.g.,~\cite{convention-1, convention-2, convention-3}. This rather ``cumbersome art'' has been gradually replaced by data-driven approaches. For parametric models of the task-learning process~\cite{memory-aug, neural-attentive}, the task-invariant ``sub-model'' can then be shared across different tasks with prior information embedded in the model weights. One typical model is that of recurrent neural networks (RNNs), where task-learning is captured by recurrent cells. However, the resultant black-box learning setup faces interpretability challenges.

As an alternative to model-committed approaches, model-agnostic meta-learning (MAML) transforms task-learning to optimizing the task-specific model parameters, while the prior amounts to initial parameters per task-level optimization, that are shared across tasks and can be learned through differentiable meta-level optimization~\cite{MAML}. Building upon MAML,  optimization-based meta-learning has been advocated to ameliorate its performance; see e.g.~\cite{metaSGD, R2D2, WarpGrad, sharp-MAML}. In addition, performance analyses have been reported to better understand the behavior of these optimization-based algorithms~\cite{bilevel-programming, convergence-MAML, global-convergence, is-Bayesian-better}. 

Interestingly, the learned initialization can be approximately viewed as the mean of an implicit Gaussian prior over the task-specific parameters~\cite{LLAMA}. Inspired by this interpretation, Bayesian methods have been advocated for meta-learning to further allow for uncertainty quantification in the model parameters. Different from its deterministic counterpart, Bayesian meta-learning seeks a prior distribution over the model parameters that best explains the data. Exact Bayesian inference however, is barely tractable as the posterior is often non-Gaussian, which prompts pursuing approximate inference methods; see e.g.,~\cite{BMAML, LLAMA, PLATIPUS, ABML}. 

MAML and its variants have appealing empirical performance, but optimizing the meta-learning loss with backpropagation is challenging due to the high-order derivatives involved. This incurs complexity that grows linearly with the number of task-level optimization steps, which renders the corresponding algorithms barely scalable. For this reason, scalability of meta-learning algorithms is of paramount importance. One remedy is to simply ignore the high-order derivatives, and rely on first-order updates only~\cite{MAML,Reptile}. Alternatively, the so-termed implicit (i)MAML relies on implicit differentiation to eliminate the explicit backpropagation. However, the proximal regularization term in iMAML is confined to be a simple isotropic Gaussian prior, which limits model expressiveness~\cite{iMAML}. 

In this paper, we develop a novel implicit Bayesian meta-learning (iBaML) approach that offers the desirable scalability, expressiveness, and performance quantification, and thus broadens the scope and appeal of meta-learning to real application domains. The contribution is threefold. 
\begin{enumerate}[i)]
\item iBaML enjoys complexity that is invariant to the number $K$ of gradient steps in task-level optimization. This fundamentally breaks the complexity-accuracy tradeoff, and makes Bayesian meta-learning affordable with more sophisticated task-level optimization algorithms. 
\item Rather than an isotropic Gaussian distribution, iBaML allows for learning more expressive priors. As a Bayesian approach, iBaML can quantify uncertainty of the estimated model parameters. 
\item Through both analytical and numerical performance studies, iBaML showcases its complexity and accuracy merits over the state-of-the-art Bayesian meta-learning methods. In a large $K$ regime, the time and space complexity can be reduced even by an order of magnitude. 
\end{enumerate}

\section{Preliminaries and problem statement}
This section outlines the meta-learning formulation in the context of supervised few-shot learning, and touches upon the associated scalability issues.

\subsection{Meta-learning setups}

Suppose we are given datasets $\mathcal{D}_t := \{ (\mathbf{x}_t^n, y_t^n) \}_{n=1}^{N_t}$, each of cardinality $|\mathcal{D}_t| = N_t$ corresponding to a task indexed by $t \in  \{ 1 ,\ldots, T \}$, where $\mathbf{x}_t^n$ is an input vector, and $y_t^n \in \mathbb{R}$ denotes its label. Set $\mathcal{D}_t$ is disjointly partitioned into a training set $\mathcal{D}_t^{\mathrm{tr}}$ and a validation set $\mathcal{D}_t^{\mathrm{val}}$, with $| \mathcal{D}_t^{\mathrm{tr}} | = N_t^{\mathrm{tr}}$ and $| \mathcal{D}_t^{\mathrm{val}} | = N_t^{\mathrm{val}}$ for $\forall t$. Typically, $N_t$ is limited, and often much smaller than what is required by supervised DL tasks. However, it is worth stressing that the number of tasks $T$ can be considerably large. Thus, $\sum_{t=1}^T N_t$ can be sufficiently large for learning a prior parameter vector shared by all tasks; e.g., using deep neural networks. 

A key attribute of meta-learning is to estimate such a task-invariant prior information parameterized by the meta-parameter $\boldsymbol{\theta}$ based on training data \emph{across} tasks. Subsequently, $\boldsymbol{\theta}$ and $\mathcal{D}_t^{\mathrm{tr}}$ are used to perform task- or \emph{inner-level optimization} to obtain the task-specific parameter $\boldsymbol{\theta}_t \in \mathbb{R}^d$. The estimate of $\boldsymbol{\theta}_t$ is then evaluated on $\mathcal{D}_t^{\mathrm{val}}$ (and potentially also $\mathcal{D}_t^{\mathrm{tr}}$) to produce a validation loss. Upon minimizing this loss summed over all the training tasks w.r.t. $\boldsymbol{\theta}$, this meta- or \emph{outer-level optimization} yields the task-invariant estimate of $\boldsymbol{\theta}$. Note that the dimension of $\boldsymbol{\theta}_t$ is not necessarily identical to that of $\boldsymbol{\theta}$; see e.g.~\cite{metaSGD, R2D2, MetaOptNet}. As we will see shortly, this nested structure can be formulated as a bi-level optimization problem. This formulation readily suggests application of meta-learning to settings such as hyperparameter tuning that also relies on a similar bi-level optimization~\cite{bilevel-programming}.

This bi-level optimization is outlined next for both deterministic and probabilistic Bayesian meta-learning variants. 

\paragraph{Optimization-based meta-learning.}
For each task $t$, let $\check{\mathcal{L}}_t^{\mathrm{tr}}(\boldsymbol{\theta}_t)$ and $\check{\mathcal{L}}_t^{\mathrm{val}} (\boldsymbol{\theta}_t)$ denote the losses over $\mathcal{D}^{\mathrm{tr}}_t$ and $\mathcal{D}^{\mathrm{val}}_t$, respectively. Further, let $\hat{\boldsymbol{\theta}}$ be the meta-parameter estimate, and $\mathcal{R} (\hat{\boldsymbol{\theta}}, \boldsymbol{\theta}_t)$ the regularizer of the learning cost per task $t$. Optimization-based meta-learning boils down to 
\begin{align}
\label{eq:opt-based-global-min}
\hat{\boldsymbol{\theta}} = 
&\argmin_{\boldsymbol{\theta}} ~\sum_{t=1}^T \check{\mathcal{L}}_t^{\mathrm{val}} (\hat{\boldsymbol{\theta}}_t (\boldsymbol{\theta})) \\
&\hspace*{-0.4cm}\mathrm{s.to}~~ \hat{\boldsymbol{\theta}}_t (\boldsymbol{\theta}) = \argmin_{\boldsymbol{\theta}_t} \check{\mathcal{L}}_t^{\mathrm{tr}} (\boldsymbol{\theta}_t) + \mathcal{R} (\boldsymbol{\theta}, \boldsymbol{\theta}_t), ~t=1,\ldots,T.\nonumber
\end{align}
The regularizer $\mathcal{R}$ can be either implicit (as in iMAML) or explicit (as in MAML).
Further, the task-invariant meta-parameter is calibrated by $\mathcal{R}$ in order to cope with overfitting. Indeed, an over-parameterized neural network could easily overfit $\mathcal{D}^{\mathrm{tr}}_t$ to produce a tiny $\check{\mathcal{L}}_t^{\mathrm{tr}}$ yet a large $\check{\mathcal{L}}_t^{\mathrm{val}}$. 

As reaching global minima can be infeasible especially with highly nonconvex neural networks, a practical alternative is an estimator $\hat{\boldsymbol{\theta}}_t$ produced by a function $\hat{\mathcal{A}}_t (\boldsymbol{\theta})$ representing an optimization algorithm, such as gradient descent (GD), with a prefixed number $K$ of iterations. Thus, a tractable version of \eqref{eq:opt-based-global-min} is 
\begin{align}
\label{eq:opt-based-subopt}
\hat{\boldsymbol{\theta}} = 
&\argmin_{\boldsymbol{\theta}} ~\sum_{t=1}^T \check{\mathcal{L}}_t^{\mathrm{val}} (\hat{\boldsymbol{\theta}}_t (\boldsymbol{\theta})) \\
&\hspace*{-0.4cm}\mathrm{s.to}~~ \hat{\boldsymbol{\theta}}_t (\boldsymbol{\theta}) = \hat{\mathcal{A}}_t (\boldsymbol{\theta}), ~~t=1,\ldots,T \nonumber
\end{align}
As an example, $\hat{\mathcal{A}}_t$ can be an one-step gradient descent initialized by $\hat{\boldsymbol{\theta}}$ with implicit priors ($\mathcal{R} (\hat{\boldsymbol{\theta}},\boldsymbol{\theta}_t) = 0$)~\cite{MAML, LLAMA}, which yields the per task parameter estimate
\begin{equation}
\label{eq:ones-step-GD}
	\hat{\boldsymbol{\theta}}_t = \hat{\mathcal{A}}_t (\boldsymbol{\theta})  = \boldsymbol{\theta} - \alpha \nabla \check{\mathcal{L}}_t^{\mathrm{tr}} (\boldsymbol{\theta}), ~~t=1,\ldots,T
\end{equation}
where $\alpha$ is the learning rate of GD, and we use the compact gradient notation 
$\nabla \check{\mathcal{L}}_t^{\mathrm{tr}} (\boldsymbol{\theta})
:= 
\nabla_{\boldsymbol{\theta}_t} \check{\mathcal{L}}_t^{\mathrm{tr}} (\boldsymbol{\theta}_t)
\big|_{\boldsymbol{\theta}_t = \boldsymbol{\theta}}$ hereafter. For later use, we also define $\mathcal{A}_t^{*}$ the (unknown) oracle function that generates the global optimum $\boldsymbol{\theta}_t^*$. 

\paragraph{Bayesian meta-learning.} The probabilistic approach to meta-learning takes a Bayesian view of the (now random) vector $\boldsymbol{\theta}_t$ per task $t$. The task-invariant vector $\boldsymbol{\theta}$ is still deterministic, and parameterizes the prior probability density function (pdf) $p(\boldsymbol{\theta}_t; \boldsymbol{\theta})$.
Task-specific learning seeks the posterior pdf $p(\boldsymbol{\theta}_t | \mathbf{y}_t^{\mathrm{tr}}; \mathbf{X}_t^{\mathrm{tr}}, \boldsymbol{\theta} )$, where $\mathbf{X}_t^{\mathrm{tr}} := [\mathbf{x}_t^1, \ldots, \mathbf{x}_t^{N_t^{\mathrm{tr}}}]$ and $\mathbf{y}_t^{\mathrm{tr}} := [y_t^1, \ldots, y_t^{N_t^{\mathrm{tr}}}]^\top$ ($^\top$ denotes transposition), while the objective per task $t$ is to maximize the conditional likelihood $p(\mathbf{y}_t^{\mathrm{val}} | \mathbf{y}_t^{\mathrm{tr}} ; \mathbf{X}_t^{\mathrm{val}}, \mathbf{X}_t^{\mathrm{tr}}, \boldsymbol{\theta})  = \int p(\mathbf{y}_t^{\mathrm{val}} | \boldsymbol{\theta}_t; \mathbf{X}_t^{\mathrm{val}}) p(\boldsymbol{\theta}_t | \mathbf{y}_t^{\mathrm{tr}}; \mathbf{X}_t^{\mathrm{tr}}, \boldsymbol{\theta}) d\boldsymbol{\theta}_t$. Along similar lines followed by its deterministic optimization-based counterpart, Bayesian meta-learning amounts to
\begin{align}
\label{eq:Bayesian-global-min}
	& \hat{\boldsymbol{\theta}}  =\argmax_{\boldsymbol{\theta}} ~\prod_{t=1}^T \int p(\mathbf{y}_t^{\mathrm{val}} | \boldsymbol{\theta}_t; \mathbf{X}_t^{\mathrm{val}}) p(\boldsymbol{\theta}_t | \mathbf{y}_t^{\mathrm{tr}}; \mathbf{X}_t^{\mathrm{tr}}, \boldsymbol{\theta} ) d\boldsymbol{\theta}_t \nonumber \\
	&\mathrm{s.to}~~~ p(\boldsymbol{\theta}_t | \mathbf{y}_t^{\mathrm{tr}}; \mathbf{X}_t^{\mathrm{tr}}, \boldsymbol{\theta} ) \propto p(\mathbf{y}_t^{\mathrm{tr}} | \boldsymbol{\theta}_t; \mathbf{X}_t^{\mathrm{tr}}) p(\boldsymbol{\theta}_t; \boldsymbol{\theta}), ~\forall t 
\end{align}
where we used that datasets are independent across tasks, and Bayes' rule in the second line. Through the posterior $p(\boldsymbol{\theta}_t | \mathbf{y}_t^{\mathrm{tr}}; \mathbf{X}_t^{\mathrm{tr}}, \boldsymbol{\theta} )$, Bayesian meta-learning quantifies the uncertainty of task-specific parameter estimate $\hat{\boldsymbol{\theta}}_t$, thus assessing model robustness. When the posterior of $\boldsymbol{\theta}_t$ is replaced by its maximum a posteriori point estimator $\hat{\boldsymbol{\theta}}_t^{\rm map}$, meaning $p(\boldsymbol{\theta}_t | \mathbf{y}_t^{\mathrm{tr}}; \mathbf{X}_t^{\mathrm{tr}}, \boldsymbol{\theta} ) = \delta_D [\boldsymbol{\theta}_t - \hat{\boldsymbol{\theta}}_t^{\rm map}]$ with $\delta_D$ denoting Dirac's delta, it turns out that~\eqref{eq:Bayesian-global-min} reduces to~\eqref{eq:opt-based-global-min}.

Unfortunately, the posterior in~\eqref{eq:Bayesian-global-min} can be intractable with nonlinear models due to the difficulty of finding analytical solutions. To overcome this, we can resort to the widely adopted approximate variational inference (VI); see e.g.~\cite{PLATIPUS, ABML, VAMPIRE}. VI searches over a family of tractable distributions for a surrogate that best matches the true posterior $p(\boldsymbol{\theta}_t | \mathbf{y}_t^{\mathrm{tr}}; \mathbf{X}_t^{\mathrm{tr}}, \boldsymbol{\theta} )$. This can be accomplished by minimizing the KL-divergence between the surrogate pdf $q(\boldsymbol{\theta}_t ; \mathbf{v}_t)$ and the true one, where $\mathbf{v}_t$ determines the variational distribution. Considering that the dimension of $\boldsymbol{\theta}_t$ can be fairly high, both the prior and surrogate posterior are often set to be Gaussian ($\cal N$) with diagonal covariance matrices. Specifically, we select the prior as $p(\boldsymbol{\theta}_t; \boldsymbol{\theta}) = \mathcal{N} (\mathbf{m}, \mathbf{D})$ with covariance $\mathbf{D} = \diag(\mathbf{d})$ and $\boldsymbol{\theta} := [\mathbf{m}^\top,\mathbf{d}^\top]^\top \in \mathbb{R}^d \times \mathbb{R}_{>0}^d$, and the surrogate posterior as $q(\boldsymbol{\theta}_t ; \mathbf{v}_t) = \mathcal{N} (\mathbf{m}_t, \mathbf{D}_t)$ with $\mathbf{D}_t = \diag(\mathbf{d}_t)$ and $\mathbf{v}_t := [\mathbf{m}_t^\top, \mathbf{d}_t^\top]^\top \in \mathbb{R}^d \times \mathbb{R}_{>0}^d$. 

To ensure tractable numerical integration over $q(\boldsymbol{\theta}_t ; \mathbf{v}_t)$, the meta-learning loss is often relaxed to an upper bound of $\sum_{t=1}^T -\log p(\mathbf{y}_t^{\mathrm{val}} | \mathbf{y}_t^{\mathrm{tr}} ; \mathbf{X}_t^{\mathrm{val}}, \mathbf{X}_t^{\mathrm{tr}}, \boldsymbol{\theta} )$. Common choices include applying Jensen’s inequality~\cite{VAMPIRE} or an extra VI~\cite{PLATIPUS, ABML} on~\eqref{eq:Bayesian-global-min}. For notational convenience, here we will denote this upper bound by $\mathcal{L}_t^{\mathrm{val}}(\mathbf{v}_t, \boldsymbol{\theta} )$. With VI and a relaxed (upper bound) objective, \eqref{eq:Bayesian-global-min} becomes 
\begin{align}
\label{eq:Bayesian-meta-VI}
& \hat{\boldsymbol{\theta}} = \argmin_{\boldsymbol{\theta}} ~ \sum_{t=1}^T \mathcal{L}_t^{\mathrm{val}} (\mathbf{v}_t^* (\boldsymbol{\theta}), \boldsymbol{\theta} ) \\
	&\mathrm{s.to}~ \mathbf{v}_t^* (\boldsymbol{\theta}) = 
\argmin_{\mathbf{v}_t} \KL \big( q(\boldsymbol{\theta}_t ; \mathbf{v}_t) \big\| p(\boldsymbol{\theta}_t | \mathbf{y}_t^{\mathrm{tr}}; \mathbf{X}_t^{\mathrm{tr}}, \boldsymbol{\theta} ) \big) ~\forall t, \nonumber
\end{align}
where $\mathcal{L}_t^{\mathrm{val}}$ depends on $\boldsymbol{\theta}$ in two ways: i) via the intermediate variable $\mathbf{v}_t^*$; and, ii) by acting directly on $\mathcal{L}_t^{\mathrm{val}}$. Note that \eqref{eq:Bayesian-meta-VI} is general enough to cover the case where $\mathcal{L}_t^{\mathrm{val}}$ is constructed using both $\mathcal{D}_t^{\mathrm{val}}$ and $\mathcal{D}_t^{\mathrm{tr}}$; see e.g.,~\cite{ABML}. 
Similar to optimization-based meta-learning, the difficulty in reaching global optima prompts one to substitute $\mathbf{v}_t^*$ with a sub-optimum $\hat{\mathbf{v}}_t$ obtained through an algorithm $\hat{\mathcal{A}}_t (\boldsymbol{\theta})$; i.e.,
\begin{align}
\label{eq:Bayesian-subopt}
& \hat{\boldsymbol{\theta}} = \argmin_{\boldsymbol{\theta}} ~ \sum_{t=1}^T \mathcal{L}_t^{\mathrm{val}} (\hat{\mathbf{v}}_t (\boldsymbol{\theta}), \boldsymbol{\theta} ) \nonumber \\
&\mathrm{s.to}~~~ \hat{\mathbf{v}}_t (\boldsymbol{\theta}) = \hat{\mathcal{A}}_t (\boldsymbol{\theta}), ~~~~t=1,\ldots,T.
\end{align}

\subsection{Scalability issues in meta-learning}
Delay and memory resources required for solving \eqref{eq:opt-based-subopt} and \eqref{eq:Bayesian-subopt} are arguably the major challenges that meta-learning faces. 
Here we will elaborate on these challenges in the optimization-based setup, but the same argument carries over to Bayesian meta-learning too. 

Consider minimizing the meta-learning loss in~\eqref{eq:opt-based-subopt} using gradient-based iteration such as Adam~\cite{Adam}. In the $(r+1)$-st iteration, 
gradients must be computed for a batch $\mathcal{B}^r \subset \{ 1, \ldots, T \}$ of tasks. Letting $\hat{\boldsymbol{\theta}}_t^r := \hat{\mathcal{A}}_t ( \hat{\boldsymbol{\theta}}^r )$, where $\hat{\boldsymbol{\theta}}^r$ denotes the meta-parameter in the $r$-th iteration, the chain rule yields the so-termed meta-gradient
\begin{equation}
\nabla_{\boldsymbol{\theta}} \check{\mathcal{L}}_t^{\mathrm{val}} (\hat{\boldsymbol{\theta}}_t^r (\boldsymbol{\theta})) \Big|_{\boldsymbol{\theta} = \hat{\boldsymbol{\theta}}^r} 
	= \nabla \hat{\mathcal{A}}_t (\hat{\boldsymbol{\theta}}^r) \nabla \check{\mathcal{L}}_t^{\mathrm{val}} (\hat{\boldsymbol{\theta}}_t^{r}),~~t \in \mathcal{B}^r
\end{equation}
where $\nabla \hat{\mathcal{A}}_t (\hat{\boldsymbol{\theta}}^r)$ contains high-order derivatives. When $\hat{\mathcal{A}}_t$ is chosen as the one-step GD (cf. \eqref{eq:ones-step-GD}), the meta-gradient is
\begin{equation}
	\nabla \hat{\mathcal{A}}_t (\hat{\boldsymbol{\theta}}^r) = \mathbf{I}_d - \alpha \nabla^2 \check{\mathcal{L}}_t^{\mathrm{tr}} (\hat{\boldsymbol{\theta}}^{r}),~~~ t \in \mathcal{B}^r. 
\end{equation}
Fortunately, in this case the meta-gradient can still be computed through the Hessian-vector product (HVP), which incurs spatio-temporal complexity $\mathcal{O}(d)$. 

In general, $\hat{\mathcal{A}}_t$ is a $K$-step GD for some $K > 1$, which gives rise to  high-order derivatives $\{ \nabla^{k} \check{\mathcal{L}}_t^{\mathrm{tr}} (\hat{\boldsymbol{\theta}}^r) \}_{k=2}^{K+1}$ in the meta-gradient. The most efficient computation of the meta-gradient calls for recursive application of HVP $K$ times, what incurs an overall complexity of $\mathcal{O}(Kd)$ in time, and $\mathcal{O}(Kd)$ in space requirements. Empirical wisdom however, favors a large $K$ because it leads to improved accuracy in approximating the true meta-gradient $\nabla_{\boldsymbol{\theta}} \check{\mathcal{L}}_t^{\mathrm{val}} (\mathcal{A}_t^{*} (\boldsymbol{\theta})) \big|_{\boldsymbol{\theta} = \hat{\boldsymbol{\theta}}^r}$. Hence, the linear increase of complexity with $K$ will impede the scaling of optimization-based meta-learning algorithms. 

When computing the meta-gradient, it should be underscored that the forward implementation of the $K$-step GD function has complexity $\mathcal{O}(Kd)$. However, the constant hidden in the ${\cal O}$ is much smaller compared to the HVP computation in the backward propagation. Typically, the constant is $1/5$ in terms of time and $1/2$ in terms of space; see ~\cite{HVP-complexity, iMAML}. For this reason, we will focus on more efficient means of obtaining the meta-gradient function $\nabla_{\boldsymbol{\theta}} \mathcal{L}_t^{\mathrm{val}} (\hat{\mathcal{A}}_t (\boldsymbol{\theta}))$ for Bayesian meta-learning. It is also worth stressing that our results in the next section will hold for an arbitrary vector $\boldsymbol{\theta} \in \mathbb{R}^d \times \mathbb{R}_{>0}^d$ instead of solely the variable $\hat{\boldsymbol{\theta}}^r$ of the $r$-th iteration. Thus, we will use the general vector $\boldsymbol{\theta}$ when introducing our approach, while we will take its value at the point $\boldsymbol{\theta} = \hat{\boldsymbol{\theta}}^r$ when presenting our meta-learning algorithm.

\section{Implicit Bayesian meta-learning}
In this section, we will first introduce the proposed implicit Bayesian meta-learning (iBaML) method, which is built on top of implicit differentiation. Then, we will provide theoretical analysis to bound and compare the errors of explicit and implicit differentiation. 

\subsection{Implicit Bayesian meta-gradients}

We start with decomposing the meta-gradient in Bayesian meta-learning \eqref{eq:Bayesian-subopt} (henceforth referred to as Bayesian meta-gradient) using the chain rule
\begin{align}
\label{eq:Bayesian-meta-grad}
	\nabla_{\boldsymbol{\theta}} \mathcal{L}_t^{\mathrm{val}} (\hat{\mathbf{v}}_t (\boldsymbol{\theta}), \boldsymbol{\theta}) 
	=&~ \nabla \hat{\mathcal{A}}_t (\boldsymbol{\theta}) 
	\nabla_1 \mathcal{L}_t^{\mathrm{val}}(\hat{\mathbf{v}}_t, \boldsymbol{\theta}) \nonumber \\
	&+ \nabla_2 \mathcal{L}_t^{\mathrm{val}}(\hat{\mathbf{v}}_t, \boldsymbol{\theta}), ~~t=1,\ldots,T
\end{align}
where $\nabla_1$ and $\nabla_2$ denote the partial derivatives of a function w.r.t. its first and second arguments, respectively. The computational burden in \eqref{eq:Bayesian-meta-grad} comes from the high-order derivatives present in the Jacobian $\nabla \hat{\mathcal{A}}_t (\boldsymbol{\theta}) $.

The key idea behind implicit differentiation is to express $\nabla \hat{\mathcal{A}}_t (\boldsymbol{\theta})$ as a function of itself, so that it can be numerically obtained without using high-order derivatives. The following lemma formalizes how the implicit Jacobian is obtained in our setup. All proofs can be found in the Appendix.
\begin{lemma}
\label{lemma:implicit_Jacobi}
Consider the Bayesian meta-learning problem in~\eqref{eq:Bayesian-meta-VI}, and let $\bar{\mathbf{v}}_t := [\bar{\mathbf{m}}_t^\top, \bar{\mathbf{d}}_t^\top]^\top$ be a local minimum of the task-level KL-divergence generated by $\bar{\mathcal{A}}_t (\boldsymbol{\theta})$. Also, let $\mathcal{L}_t^{\mathrm{tr}}(\mathbf{v}_t) := \mathbb{E}_{q(\boldsymbol{\theta}_t; \mathbf{v}_t)} [ -\log p(\mathbf{y}_t^{\mathrm{tr}} | \boldsymbol{\theta}_t; \mathbf{X}_t^{\mathrm{tr}}) ]$ denote the expected negative log-likelihood (nll) on $\mathcal{D}_t^{\mathrm{tr}}$. If $\mathbf{H}_t (\bar{\mathbf{v}}_t) := \nabla^2 \mathcal{L}_t^{\mathrm{tr}} (\bar{\mathbf{v}}_t) + \left[ \begin{matrix}
	\mathbf{D}^{-1} & \mathbf{0}_d \\
	\mathbf{0}_d & \frac{1}{2} \big( \mathbf{D}^{-1} + 2 \diag \big( \nabla_{\bar{\mathbf{d}}_t} \mathcal{L}_t^{\mathrm{tr}} (\bar{\mathbf{v}}_t) \big) \big)^2
	\end{matrix} \right]$ is invertible, then it holds for $\forall t \in \{ 1, \ldots, T \}$ that
\begin{align}
	&\nabla \bar{\mathcal{A}}_t (\boldsymbol{\theta}) = \nonumber \\
	&\left[ \begin{matrix}
	\mathbf{D}^{-1} & \mathbf{0}_d \\
	-\diag \big( \nabla_{\bar{\mathbf{m}}_t} \mathcal{L}_t^{\mathrm{tr}}(\bar{\mathbf{v}}_t) \big) \mathbf{D}^{-1} & \frac{1}{2}\mathbf{D}^{-2}
	\end{matrix} \right]
	\mathbf{H}_t^{-1} (\bar{\mathbf{v}}_t).
\end{align}
\end{lemma}

Two remarks are now in order regarding the technical assumption, and connections with iMAML. For notational brevity, define the block matrix 
\begin{equation}
\label{eq:def-G}
\mathbf{G}_t(\bar{\mathbf{v}}_t) := \left[ \begin{matrix}
		\mathbf{D}^{-1} & \mathbf{0}_d \\
		-\diag \big( \nabla_{\bar{\mathbf{m}}_t} \mathcal{L}_t^{\mathrm{tr}}(\bar{\mathbf{v}}_t) \big) \mathbf{D}^{-1} & \frac{1}{2}\mathbf{D}^{-2}
	\end{matrix} \right].
\end{equation}

\begin{remark}
\normalfont The invertibility of $\mathbf{H}_t (\bar{\mathbf{v}}_t)$ in Lemma~\ref{lemma:implicit_Jacobi} is assumed to ensure uniqueness of $\nabla \bar{\mathcal{A}}_t (\boldsymbol{\theta})$. Without this assumption, it turns out that $\bar{\mathbf{v}}_t$ can be a singular point, belonging to a subspace where any point is also a local minimum. The Bayesian meta-gradients~\eqref{eq:Bayesian-meta-grad} of the points in this subspace form a set
\begin{align}
	\bar{\mathcal{G}}_{t} = 
	\Big\{ 
	& \mathbf{G}_t(\bar{\mathbf{v}}_t) \big( \mathbf{H}_t^{\dagger} (\bar{\mathbf{v}}_t) \nabla_1 \mathcal{L}_t^{\mathrm{val}}(\bar{\mathbf{v}}_t, \boldsymbol{\theta})
	+ \mathbf{u} \big) \nonumber \\
	& + \nabla_2 \mathcal{L}_t^{\mathrm{val}}(\bar{\mathbf{v}}_t, \boldsymbol{\theta}) ~\big|~ \forall \mathbf{u} \in \mathrm{Null} \big( \mathbf{H}_t (\bar{\mathbf{v}}_t) \big)
	\Big\}
\end{align}
where $^{\dagger}$ represents pseudo-inverse, and $\mathrm{Null}(\cdot)$ stands for the null space. Upon replacing $\mathbf{H}_t^{-1} (\bar{\mathbf{v}}_t)$ with $\mathbf{H}_t^{\dagger} (\bar{\mathbf{v}}_t)$, one can generalize Lemma~\ref{lemma:implicit_Jacobi}, and forgo the invertibility assumption. 
\end{remark}

\begin{remark}
\normalfont To recognize how Lemma~\ref{lemma:implicit_Jacobi} links iBaML with iMAML~\cite{iMAML}, consider the special case where the covariance matrices of the prior and local minimum are fixed as $\mathbf{D} \equiv \lambda^{-1} \mathbf{I}_d$ and $\bar{\mathbf{D}}_t \equiv \mathbf{0}_d$ for some constant $\lambda$. Since $\mathbf{d} = [\lambda^{-1}, \ldots, \lambda^{-1}] \in \mathbb{R}^d$ is a constant vector, Lemma~\ref{lemma:implicit_Jacobi} boils down to
\begin{align}
	\nabla_{\mathbf{m}} \bar{\mathcal{A}}_t (\boldsymbol{\theta}) 
	&= \mathbf{D}^{-1} \big( \nabla_{\mathbf{m}}^2 \mathcal{L}_t^{\mathrm{tr}} (\bar{\mathbf{v}}_t) + \mathbf{D}^{-1} \big)^{-1} \nonumber \\
	&= \big( \lambda^{-1} \nabla_{\mathbf{m}}^2 \mathcal{L}_t^{\mathrm{tr}} (\bar{\mathbf{v}}_t) + \mathbf{I}_d \big)^{-1}
\end{align}
which coincides with Lemma 1 of~\cite{iMAML}. Hence, iBaML subsumes iMAML whose expressiveness is confined because $\mathbf{d}$ is fixed, while iBaML entails a  learnable covariance matrix in the prior $p(\boldsymbol{\theta}_t; \boldsymbol{\theta})$. In addition, the uncertainty of iMAML's training posterior $p(\boldsymbol{\theta}_t | \mathbf{y}_t^{\mathrm{tr}}; \mathbf{X}_t^{\mathrm{tr}}, \boldsymbol{\theta})$ can be more challenging to quantify than that in iBaML.
\end{remark}

\begin{algorithm}[t]
	\caption{Implicit Bayesian meta-learning (iBaML)}
	\label{alg:iBaML}
	\begin{algorithmic}[1]
	\STATE \textbf{Inputs:} tasks $\{ 1, \ldots, T \}$ with their $\mathcal{D}_t^{\mathrm{tr}}$ and $\mathcal{D}_t^{\mathrm{val}}$, and meta-learning rate $\beta$. \\
	\STATE \textbf{Initialization:} initialize $\hat{\boldsymbol{\theta}}^0$ randomly, and iteration counter $r = 0$.
	\REPEAT
		\STATE Sample a batch $\mathcal{B}^r \subset \{ 1, \dots, T\}$ of tasks;
		\FOR{$t \in \mathcal{B}^r$}
			\STATE Compute task-level sub-optimum $\hat{\mathbf{v}}_t^r = \hat{\mathcal{A}}_t (\hat{\boldsymbol{\theta}}^r)$ using e.g. $K$-step GD;
			\STATE Approximate $\hat{\mathbf{u}}_t^r \approx \mathbf{H}_t^{-1} (\hat{\mathbf{v}}_t^r) \nabla_1 \mathcal{L}_t^{\mathrm{val}}(\hat{\mathbf{v}}_t^r, \hat{\boldsymbol{\theta}}^r)$ with $L$-step CG;
			\STATE Compute meta-level gradient $\hat{\mathbf{g}}_t^r = \mathbf{G}_t (\hat{\mathbf{v}}_t^r) \hat{\mathbf{u}}_t^r + \nabla_2 \mathcal{L}_t^{\mathrm{val}}(\hat{\mathbf{v}}_t^r, \hat{\boldsymbol{\theta}}^r)$ using~\eqref{eq:implicit-grad};
		\ENDFOR 
		\STATE Update $\hat{\boldsymbol{\theta}}^{r+1} = \hat{\boldsymbol{\theta}}^{r} - \beta \frac{1}{|\mathcal{B}^r|} \sum_{t \in \mathcal{B}^r} \hat{\mathbf{g}}_t^r$; 
		\STATE $r = r + 1$;
	\UNTIL{convergence}
	\STATE \textbf{Output:} $\hat{\boldsymbol{\theta}}^{r}$.
	\end{algorithmic}
\end{algorithm}

An immediate consequence of Lemma~\ref{lemma:implicit_Jacobi} is the so-called generalized implicit gradients. Suppose that $\hat{\mathcal{A}}_t$ involves a $K$ sufficiently large for the sub-optimal point $\hat{\mathbf{v}}_t$ to be close to a local optimum $\bar{\mathbf{v}}_t$. The Bayesian meta-gradient~\eqref{eq:Bayesian-meta-grad} can then be approximated through
\begin{align}
\label{eq:approx-Bayesian-meta-grad}
	&\nabla_{\boldsymbol{\theta}} \mathcal{L}_t^{\mathrm{val}} (\hat{\mathbf{v}}_t (\boldsymbol{\theta}), \boldsymbol{\theta}) \\
	&\approx \mathbf{G}_t (\hat{\mathbf{v}}_t) \mathbf{H}_t^{-1} (\hat{\mathbf{v}}_t)
	\nabla_1 \mathcal{L}_t^{\mathrm{val}}(\hat{\mathbf{v}}_t, \boldsymbol{\theta}) + \nabla_2 \mathcal{L}_t^{\mathrm{val}}(\hat{\mathbf{v}}_t, \boldsymbol{\theta}), ~\forall t. \nonumber
\end{align}
The approximate implicit gradient in \eqref{eq:approx-Bayesian-meta-grad} is computationally expensive due to the matrix inversion $\mathbf{H}_t^{-1} (\hat{\mathbf{v}}_t)$, which incurs complexity $\mathcal{O}(d^3)$. To relieve the computational burden, a key observation is that $\mathbf{H}_t^{-1} (\hat{\mathbf{v}}_t) \nabla_1 \mathcal{L}_t^{\mathrm{val}}(\hat{\mathbf{v}}_t, \boldsymbol{\theta})$ is the solution of the optimization problem
\begin{equation}
\label{eq:CG-opt-prob}
	\argmin_{\mathbf{u}} \frac{1}{2} \mathbf{u}^\top \mathbf{H}_t (\hat{\mathbf{v}}_t) \mathbf{u} - \mathbf{u}^\top \nabla_1 \mathcal{L}_t^{\mathrm{val}}(\hat{\mathbf{v}}_t, \boldsymbol{\theta}).
\end{equation}
Given that the square matrix $\mathbf{H}_t (\hat{\mathbf{v}}_t)$ is by definition symmetric, problem~\eqref{eq:CG-opt-prob} can be efficiently solved using the conjugate gradient (CG) iteration. Specifically, the complexity of CG is dominated by the matrix-vector product $\mathbf{H}_t (\hat{\mathbf{v}}_t) \mathbf{p}$ (for some vector $\mathbf{p} \in \mathbb{R}^{2d}$), given by
\begin{align}\label{eq:CG-Ht_P}
	\mathbf{H}_t (\hat{\mathbf{v}}_t) \mathbf{p} 
	&= \nabla^2 \mathcal{L}_t^{\mathrm{tr}} (\hat{\mathbf{v}}_t) \mathbf{p}  \\
	&~+ \left[ \begin{matrix}
		\mathbf{D}^{-1} & \mathbf{0}_d \\
		\mathbf{0}_d & \frac{1}{2} \big( \mathbf{D}^{-1} + 2 \diag \big( \nabla_{\hat{\mathbf{d}}_t} \mathcal{L}_t^{\mathrm{tr}} (\hat{\mathbf{v}}_t) \big) \big)^2 
	\end{matrix} \right] \mathbf{p}. \nonumber 
\end{align}
The first term on the right-hand side of (\ref{eq:CG-Ht_P}) is an HVP, and the second is the multiplication of a diagonal matrix with a vector. Note that with the diagonal matrix, the latter term boils down to a dot product, implying that the complexity of each CG iteration is as low as $\mathcal{O}(d)$. In practice, a small number of CG iterations suffices to produce an accurate estimate of $\mathbf{H}_t^{-1} (\hat{\mathbf{v}}_t) \nabla_1 \mathcal{L}_t^{\mathrm{val}}(\hat{\mathbf{v}}_t, \boldsymbol{\theta})$ thanks to its fast  convergence rate~\cite{CG-convergence-1,CG-convergence-2}. In order to control the total complexity of iBaML, we set the maximum number of CG iterations to a constant $L$. 

Having obtained an approximation of the matrix-inverse-vector product $\mathbf{H}_t^{-1} (\hat{\mathbf{v}}_t) \nabla_1 \mathcal{L}_t^{\mathrm{val}}(\hat{\mathbf{v}}_t, \boldsymbol{\theta})$, we proceed to estimate the Bayesian meta-gradient. Let $\hat{\mathbf{u}}_t := [ \hat{\mathbf{u}}_{t, \mathbf{m}}^\top, \hat{\mathbf{u}}_{t, \mathbf{d}}^\top ]^\top$ be the output of the CG method with subvectors $\hat{\mathbf{u}}_{t, \mathbf{m}},~\hat{\mathbf{u}}_{t, \mathbf{d}} \in \mathbb{R}^d$. Then, it follows from~\eqref{eq:approx-Bayesian-meta-grad} that
\begin{align}
\label{eq:implicit-grad}
	&\nabla_{\boldsymbol{\theta}} \mathcal{L}_t^{\mathrm{val}} (\hat{\mathbf{v}}_t (\boldsymbol{\theta}), \boldsymbol{\theta}) \nonumber \\
	&\approx \mathbf{G}_t (\hat{\mathbf{v}}_t) \hat{\mathbf{u}}_t 
	+ \nabla_2 \mathcal{L}_t^{\mathrm{val}}(\hat{\mathbf{v}}_t, \boldsymbol{\theta}) \nonumber \\
	&=\left[ \begin{matrix}
		\mathbf{D}^{-1} \hat{\mathbf{u}}_{t,\mathbf{m}} \\
		-\diag \big( \nabla_{\hat{\mathbf{m}}_t} \mathcal{L}_t^{\mathrm{tr}} (\hat{\mathbf{v}}_t) \big) \mathbf{D}^{-1} \hat{\mathbf{u}}_{t,\mathbf{m}} + \frac{1}{2} \mathbf{D}^{-2} \hat{\mathbf{u}}_{t, \mathbf{d}}
	\end{matrix} \right] \nonumber \\
	&~~~~~~+ \nabla_2 \mathcal{L}_t^{\mathrm{val}}(\hat{\mathbf{v}}_t, \boldsymbol{\theta}) \nonumber := \hat{\mathbf{g}}_t, ~~~~t=1, \dots ,T 
\end{align}
where we also used the definition~\eqref{eq:def-G}. Again, the diagonal-matrix-vector products in~\eqref{eq:implicit-grad} can be efficiently computed through dot products, which incur complexity $\mathcal{O}(d)$. The step-by-step pseudocode of the iBaML is listed under Algorithm~\ref{alg:iBaML}. 

In a nutshell, the implicit Bayesian meta-gradient computation consumes  $\mathcal{O}(Ld)$ time, regardless of the optimization algorithm $\hat{\mathcal{A}}_t$. One can even employ more complicated algorithms such as second-order matrix-free optimization~\cite{K-FAC, GGN}. In addition, as the time complexity does not depend on $K$, one can increase $K$ to reduce the approximation error in \eqref{eq:approx-Bayesian-meta-grad}. The space complexity of iBaML is only $\mathcal{O} (d)$ thanks to the iterative implementation of CG steps. These considerations explain how iBaML addresses the scalability issue of explicit backpropagation. 

\subsection{Theoretical analysis}
This section deals with performance analysis of both explicit and implicit gradients in Bayesian meta-learning to further understand their differences. Similar to~\cite{iMAML}, our results will rely on the following assumptions.

\begin{assumption}
\label{as:local-min}
Vector	$\bar{\mathbf{v}}_t = \bar{\mathcal{A}}_t (\boldsymbol{\theta})$ is a local minimum of the KL-divergence in~\eqref{eq:Bayesian-meta-VI}.
\end{assumption}

\begin{assumption}
\label{as:meta-loss}
	The meta-loss function $\mathcal{L}_t^{\mathrm{val}} (\mathbf{v}_t, \boldsymbol{\theta})$ is $A_t$-Lipschitz and $B_t$-smooth w.r.t. $\mathbf{v}_t$ while its partial gradient $\nabla_2 \mathcal{L}_t^{\mathrm{val}} (\mathbf{v}_t, \boldsymbol{\theta})$ is $C_t$-Lipschitz w.r.t. $\mathbf{v}_t$. 
\end{assumption}

\begin{assumption}
\label{as:task-nll}
	The expected nll function $\mathcal{L}_t^{\mathrm{tr}}(\mathbf{v}_t)$ is $D_t$-smooth, and has a Hessian that is $E_t$-Lipschitz.
\end{assumption}

\begin{assumption}
\label{as:invertible}
Matrices $\mathbf{H}_t (\hat{\mathbf{v}}_t)$ and $\mathbf{H}_t (\bar{\mathbf{v}}_t)$ are both non-singular; that is, their smallest singular value $\sigma_t := \min \big\{ \sigma_{\min} \big( \mathbf{H}_t (\hat{\mathbf{v}}_t) \big), \sigma_{\min} \big( \mathbf{H}_t (\bar{\mathbf{v}}_t) \big) \big\} > 0$. 
\end{assumption}

\begin{assumption}
\label{as:bounded-var}
    Prior variances are positive and bounded, meaning $0 < D_{\min} \le [\mathbf{d}]_i \le D_{\max}, ~i = 1, \ldots, d$. 
\end{assumption}
Based on these assumptions, we can establish the following result. 
\begin{theorem}[Explicit Bayesian meta-gradient error bound]
\label{theor:explicit}
Consider the Bayesian meta-learning problem~\eqref{eq:Bayesian-subopt}. Let $\epsilon_t := \| \hat{\mathbf{v}}_t - \bar{\mathbf{v}}_t \|_2$ be the task-level optimization error, and $\delta_t := \| \nabla \hat{\mathcal{A}}_t (\boldsymbol{\theta}) - \mathbf{G}_t (\hat{\mathbf{v}}_t) \mathbf{H}_t^{-1} (\hat{\mathbf{v}}_t) \|_2$ the error in the Jacobian. Upon defining $\rho_t := \max \big\{ \| \nabla_{\bar{\mathbf{v}}_t} \mathcal{L}_t^{\mathrm{tr}} (\bar{\mathbf{v}}_t) \|_{\infty}, \| \nabla_{\hat{\mathbf{v}}_t} \mathcal{L}_t^{\mathrm{tr}} (\hat{\mathbf{v}}_t) \|_{\infty} \big\}$, and with Assumptions~\ref{as:local-min}-\ref{as:bounded-var} in effect, it holds for $t \in \{ 1,\ldots, T \}$ that
\begin{align}
	&\big\| \nabla_{\boldsymbol{\theta}} \mathcal{L}_t^{\mathrm{val}} \big( \hat{\mathbf{v}}_t (\boldsymbol{\theta}), \boldsymbol{\theta} \big) 
	-\nabla_{\boldsymbol{\theta}} \mathcal{L}_t^{\mathrm{val}} \big( \bar{\mathbf{v}}_t (\boldsymbol{\theta}), \boldsymbol{\theta} \big) \big\|_2 \nonumber \\
	&~~~~~~~~~~~~~~~~~~~~~~~~~~~~~~~~~~~~~~~~~~~~~~~~~~~~~~~~~~~\le F_t \epsilon_t + A_t \delta_t
\end{align}
where $F_t$ is a constant dependent on $\rho_t$. 
\end{theorem}

Theorem~\ref{theor:explicit} asserts that the $\ell_2$ error of the explicit Bayesian meta-gradient relative to the true depends on the task-level optimization error as well as the error in the Jacobian, where the former captures the Euclidean distance of the local minimum $\bar{\mathbf{v}}_t$ and its approximation $\hat{\mathbf{v}}_t$, while the latter characterizes how the sub-optimal function $\hat{\mathcal{A}}_t$ influences the Jacobian. Both errors can be reduced by increasing $K$ in the task-level optimization, at the cost of time and space complexity for backpropagating $\nabla \hat{\mathcal{A}}_t (\boldsymbol{\theta})$. Ideally, one can have $\delta_t = 0$ when $\hat{\mathbf{v}}_t$ is a local optimum, and $\epsilon_t = 0$ when choosing $\bar{\mathbf{v}}_t = \hat{\mathbf{v}}_t$.

Next, we derive an error bound for implicit differentiation.

\begin{theorem}[Implicit Bayesian meta-gradient error bound]
\label{theor:implicit}
Consider the Bayesian meta-learning problem~\eqref{eq:Bayesian-subopt}. Let $\epsilon_t := \| \hat{\mathbf{v}}_t - \bar{\mathbf{v}}_t \|_2$ be the task-level optimization error, and $\delta_t' := \| \hat{\mathbf{u}}_t - \mathbf{H}_t^{-1} (\hat{\mathbf{v}}_t) \nabla_1 \mathcal{L}_t^{\mathrm{val}} (\hat{\mathbf{v}}_t, \boldsymbol{\theta}) \|$  the CG error. Upon defining $\rho_t := \max \big\{ \| \nabla_{\bar{\mathbf{v}}_t} \mathcal{L}_t^{\mathrm{tr}} (\bar{\mathbf{v}}_t) \|_{\infty}, \| \nabla_{\hat{\mathbf{v}}_t} \mathcal{L}_t^{\mathrm{tr}} (\hat{\mathbf{v}}_t) \|_{\infty} \big\}$, and with Assumptions~\ref{as:local-min}-\ref{as:bounded-var} in effect, it holds for $t \in \{ 1,\ldots, T \}$ that
\begin{equation}
	\big\| \hat{\mathbf{g}}_t -\nabla_{\boldsymbol{\theta}} \mathcal{L}_t^{\mathrm{val}} \big( \bar{\mathbf{v}}_t (\boldsymbol{\theta}), \boldsymbol{\theta} \big) \big\|_2 \le F_t' \epsilon_t + G_t' \delta_t',
\end{equation}
where $F_t'$ and $G_t'$ are constants dependent on $\rho_t$. 
\end{theorem}

While the bound on implicit meta-gradient also depends on the task-level optimization error, the difference with Theorem~\ref{theor:explicit} is highlighted in the CG error. The fast convergence of CG leads to a tolerable $\delta_t'$ even with a small $L$. As a result, one can opt for a large $K$ to reduce task-level optimization error $\epsilon_t$, and a small $L$ to obtain a satisfactory approximation of the meta-gradient. 

 It is worth stressing that $\bar{\mathbf{v}}_t$ in Theorems~\ref{theor:explicit} and~\ref{theor:implicit} can denote \emph{any} local optimum. It further follows by definition that both $\delta_t$ and $\delta_t'$ do not rely on the choice of local optima, yet $\epsilon_t$ does. One final remark is now in order. 

\begin{remark}
\normalfont Theorems~\ref{theor:explicit} and~\ref{theor:implicit} can be further simplified under the additional assumption that $\mathcal{L}_t^{\mathrm{tr}} (\mathbf{v}_t)$ is $H_t$-Lipschitz. In such a case, we have $\rho_t \le H_t$, and thus the scalars $F_t$, $F_t'$ and $G_t'$ boil down to task-specific constants. 
\end{remark}


\section{Numerical tests}
Here we test and showcase on synthetic and real data the analytical novelties of this contribution. Our implementation relies on the PyTorch~\cite{pytorch}, and codes are available at \url{https://github.com/zhangyilang/iBaML}. 

\subsection{Synthetic data}
Here we experiment on the errors between explicit and implicit gradients
over a synthetic dataset. The data are generated using the Bayesian linear regression model
\begin{equation}
	y_t^n = \langle \boldsymbol{\theta}_t, \mathbf{x}_t^n \rangle + e_t^n,~\forall n,~~~t=1,\ldots,T
\end{equation}
where $\{ \boldsymbol{\theta}_t \}_{t=1}^T$ are i.i.d. samples drawn from a distribution $p(\boldsymbol{\theta}_t; \hat{\boldsymbol{\theta}})$ that is unknown during meta-training, and $e_t^n$ is the additive white Gaussian noise (AWGN) with known variance $\sigma^2$. Although the current training posterior $p(\boldsymbol{\theta}_t | y_t^{\mathrm{tr}}; \mathbf{X}_t^{\mathrm{tr}}, \boldsymbol{\theta} )$ becomes tractable, we still focus on the VI approximation for uniformity. Within this rudimentary linear case, it can be readily verified that the task-level optimum $\mathbf{v}_t^* := [\mathbf{m}_t^{*\top}, \mathbf{d}_t^{*\top}]^\top$ of~\eqref{eq:Bayesian-meta-VI} is given by
\begin{subequations}
\begin{align}
\label{eq:linear-opt-mean}
&	\mathbf{m}_t^* = \Big( \frac{1}{\sigma^2} \mathbf{X}_t^{\mathrm{tr}} (\mathbf{X}_t^{\mathrm{tr}} )^\top + \mathbf{D}^{-1} \Big)^{-1} \big( \mathbf{D}^{-1} \mathbf{m} + \frac{1}{\sigma^2} \mathbf{X}_t^{\mathrm{tr}} \mathbf{y}_t^{\mathrm{tr}} \big) \\
\label{eq:linear-opt-var}
&	\mathbf{d}_t^* = \Big( \frac{1}{2\sigma^2} \diag \big( \mathbf{X}_t^{\mathrm{tr}} (\mathbf{X}_t^{\mathrm{tr}})^\top \big) + \mathbf{d}^{-1} \Big)^{-1},~~ t=1\ldots,T
\end{align}
\end{subequations}
where $\diag(\mathbf{M})$ is a vector collecting the diagonal entries of matrix $\mathbf{M}$. The true posterior in the linear case is $p(\boldsymbol{\theta}_t | \mathbf{y}_t^{\mathrm{tr}}; \mathbf{X}_t^{\mathrm{tr}}, \boldsymbol{\theta} ) = \mathcal{N} (\mathbf{m}_t^*, \big( \frac{1}{2\sigma^2} ( \mathbf{X}_t^{\mathrm{tr}} (\mathbf{X}_t^{\mathrm{tr}})^\top ) + \mathbf{d}^{-1} \big)^{-1})$, implying that the posterior covariance matrix is essentially approximated by its diagonal counterpart $\mathbf{D}_t^*$ in VI. Lemma~\ref{lemma:implicit_Jacobi} and~\eqref{eq:Bayesian-meta-grad} imply that the oracle meta-gradient is
\begin{align}
	&\nabla_{\boldsymbol{\theta}} \mathcal{L}_t^{\mathrm{val}} (\mathbf{v}_t^* (\boldsymbol{\theta}), \boldsymbol{\theta}) \\
	& = \mathbf{G}_t (\mathbf{v}_t^*) \mathbf{H}_t^{-1} (\mathbf{v}_t^*) \nabla_1 \mathcal{L}_t^{\mathrm{val}}(\mathbf{v}_t^*, \boldsymbol{\theta}) 
	+ \nabla_2 \mathcal{L}_t^{\mathrm{val}}(\mathbf{v}^*_t, \boldsymbol{\theta}),~\forall t. \nonumber
\end{align}

As a benchmark meta-learning algorithm, we selected the amortized Bayesian meta-learning (ABML) in~\cite{ABML}. The metric used for performance assessment is the normalized root-mean-square error (NRMSE) between the true meta-gradient $\nabla_{\boldsymbol{\theta}} \mathcal{L}_t^{\mathrm{val}} (\mathbf{v}_t^* (\boldsymbol{\theta}), \boldsymbol{\theta})$, and the estimated meta-gradients $\nabla_{\boldsymbol{\theta}} \mathcal{L}_t^{\mathrm{val}} (\hat{\mathbf{v}}_t (\boldsymbol{\theta}), \boldsymbol{\theta})$ and $\hat{\mathbf{g}}_t$; see also the Appendix for additional details on the numerical test.

Figure~\ref{fig:linear_err} depicts the NRMSE as a function of $K$ for the first iteration of ABML, that is at the point $\boldsymbol{\theta} = \hat{\boldsymbol{\theta}}^0$. For explicit and implicit gradients, the NRMSE decreases as $K$ increases, while the former outperforms the latter for $K \le 5$, and the vice-versa for $K>5$. These observations confirm our analytical results. Intuitively, factors $F_t \epsilon_t$ and $F_t' \epsilon_t$ caused by imprecise task-level optimization dominate the upper bounds for small $K$, thus resulting in large NRMSE. Besides, implicit gradients are more sensitive to task-level optimization errors. One conjecture is that iBaML is developed based on Lemma~\ref{lemma:implicit_Jacobi}, where the matrix inversion can be sensitive to $\bar{\mathbf{v}}_t$'s variation. Despite that the conditioning number $\kappa$ of $\mathbf{X}_t^{\mathrm{tr}}$ takes on a large value purposely so that $\epsilon_t$ decreases slowly with $K$, a small $K$ suffices to capture accurately implicit gradients. The main reason is that the CG error $\delta_t'$ can become sufficiently small even with only $L = 2$ steps, while $\delta_t$ remains large because GD converges slowly. 

\begin{figure}[t]
	\centering
	\includegraphics[width=.9\columnwidth]{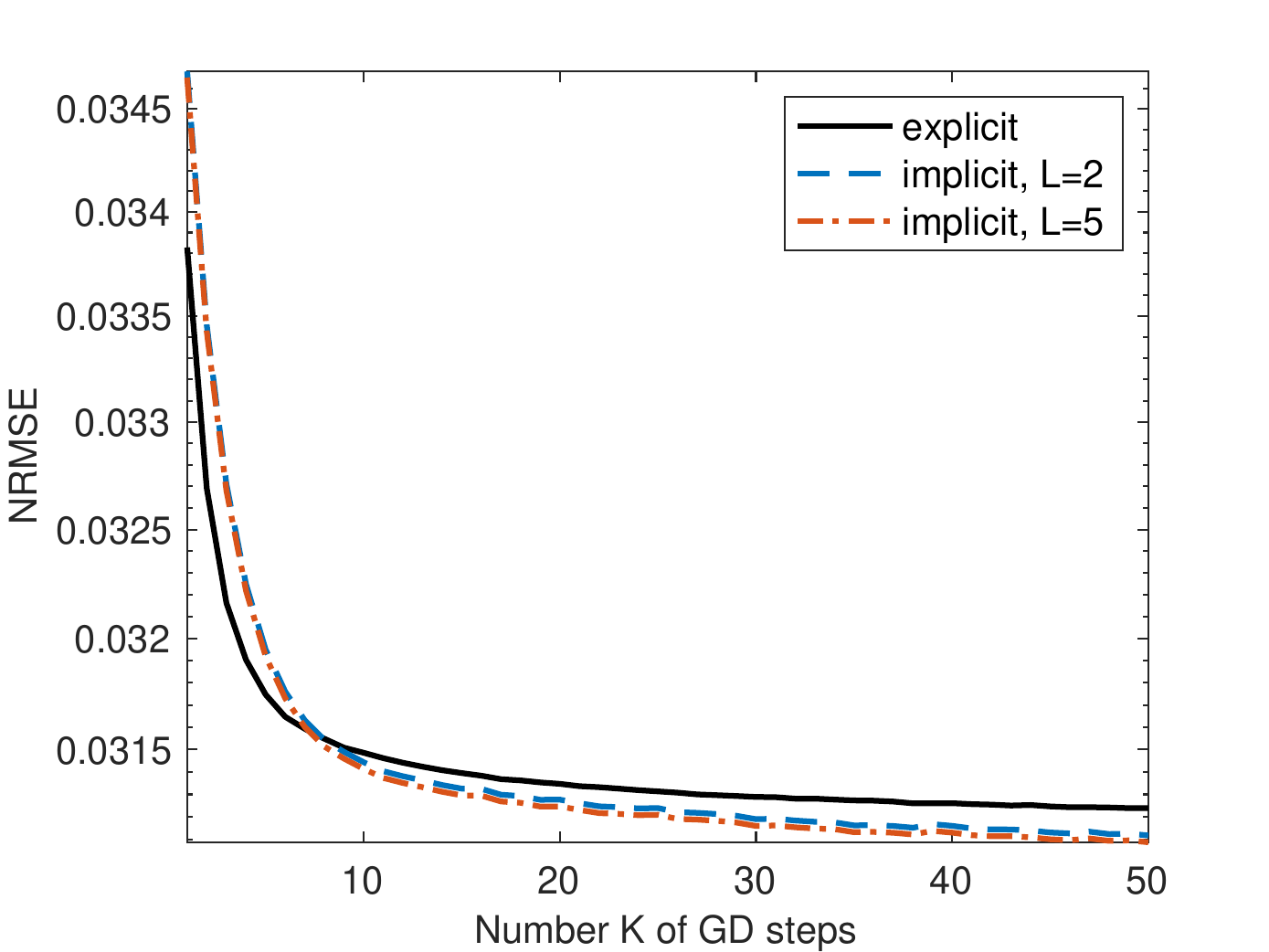}
    \vspace{-5pt}
	\caption{Gradient error comparison on synthetic dataset.}
	\label{fig:linear_err}
    \vspace{-5pt}
\end{figure}

\subsection{Real data}

\begin{figure*}[t]
\vspace{-5pt}
	\centering
	\subfloat[Time complexity]{
		\includegraphics[width=0.45\textwidth]{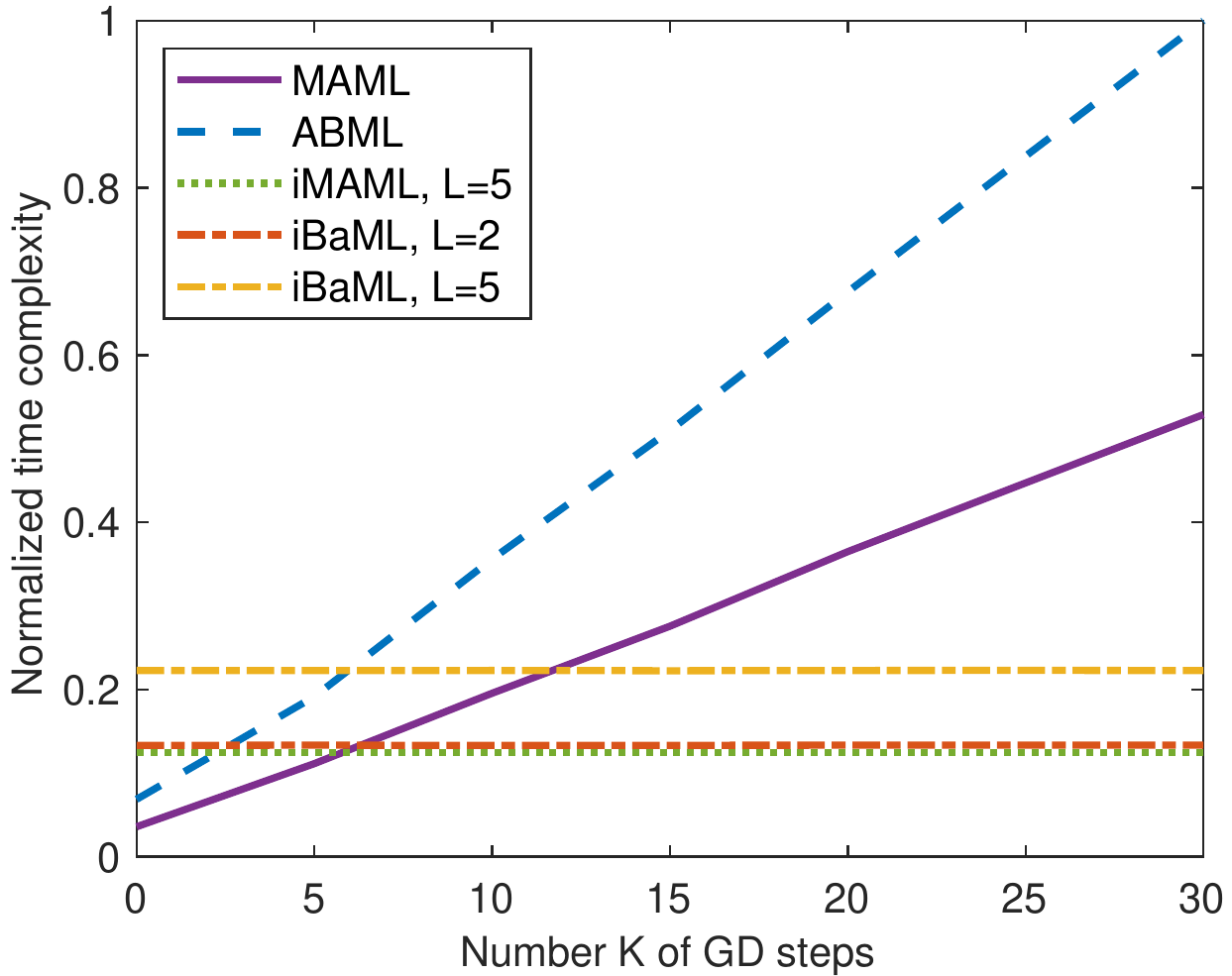}
		}
	\subfloat[Space complexity]{
		\includegraphics[width=0.45\textwidth]{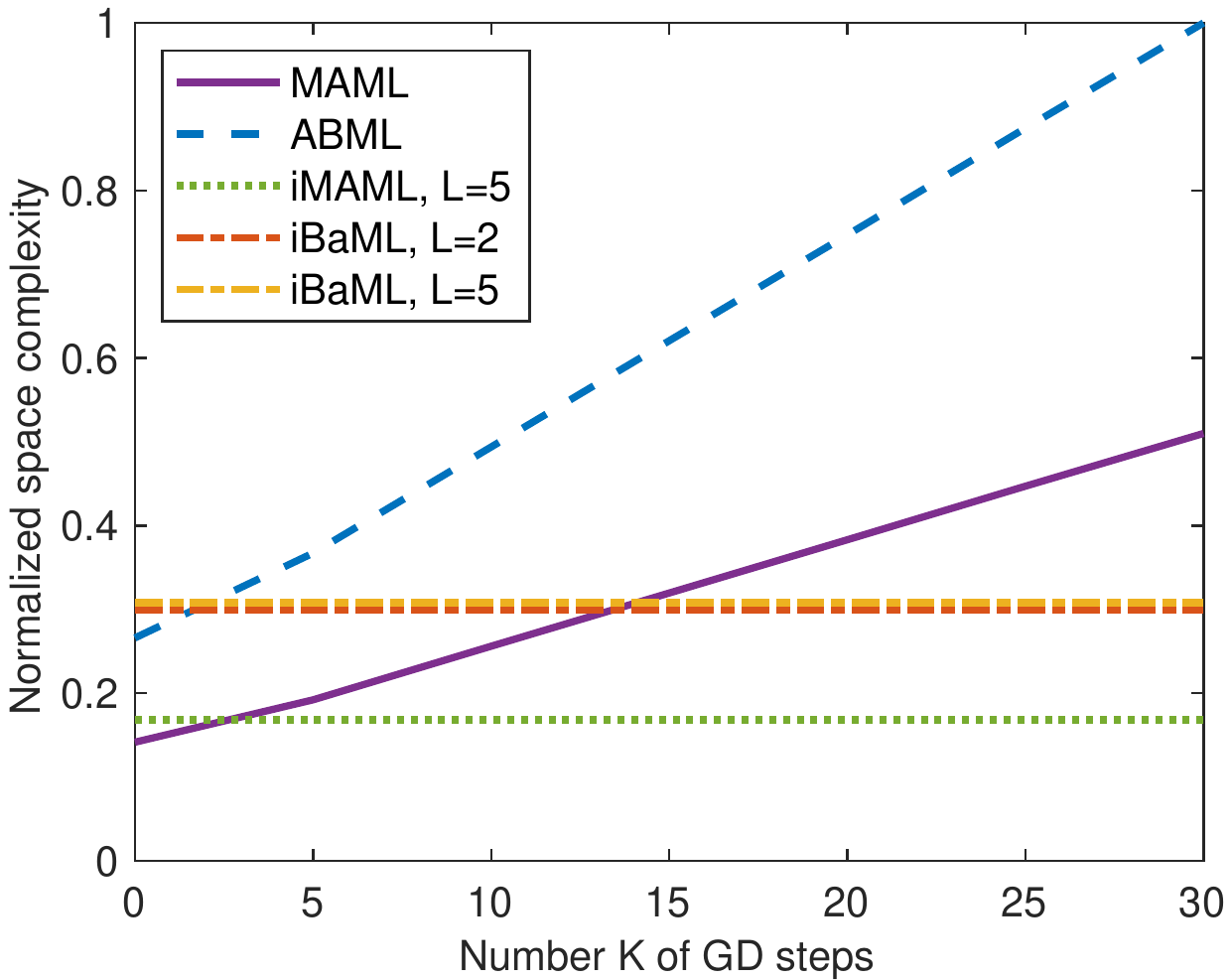}
		}
    \vspace{-5pt}
    \caption{Time and space complexity comparisons for meta-gradients computation on $5$-class $1$-shot \textit{mini}ImageNet dataset.}
	\label{fig:complexity}
\end{figure*}
Next, we conduct tests to assess the performance of iBaML on real datasets. We consider one of the most widely used few-shot dataset for classification  \textit{mini}ImageNet~\cite{MatchingNets}. This dataset consists of natural images categorized in $100$ classes, with $600$ samples per class. All images are cropped to have size of $84 \times 84$. We adopt the dataset splitting suggested by~\cite{opt-as}, where $64$, $16$ and $20$ disjoint classes are used for meta-training, meta-validation and meta-testing, respectively. The setups of the numerical test follow from the standard $W$-class $S^{\mathrm{tr}}$-shot few-shot learning protocol in~\cite{MatchingNets}. In particular, each task has $W$ randomly selected classes, and each class contains $S^{\mathrm{tr}}$ training images and $S^{\mathrm{val}}$ validation images. In other words, we have $N^{\mathrm{tr}} = S^{\mathrm{tr}} W$ and $N^{\mathrm{val}} = S^{\mathrm{val}} W$. We further adopt the typical choices with $W = 5$, $S^{\mathrm{tr}} \in \{ 1, 5 \}$, and $S^{\mathrm{val}} = 15$. It should be noted that the training and validation sets are also known as support and query sets in the context of few-shot learning.

We first empirically compare the computational complexity (time and space) for explicit versus implicit gradients on the $5$-class $1$-shot \textit{mini}ImageNet dataset. Here we are only interested in backward complexity, so the delay and memory requirements for forward pass of $\hat{\mathcal{A}}_t$ is excluded. Figure~\ref{fig:complexity}(a) plots the time complexity of explicit and implicit gradients against $K$. It is observed that the time complexity of explicit gradient grows linearly with $K$, while the implicit one increases only with $L$ but not $K$. Moreover, the explicit and implicit gradients have comparable time complexity when $K = L$. As far as space complexity, Figure~\ref{fig:complexity}(b) illustrates that memory usage with explicit gradients is proportional to $K$. In contrast, the memory used in the implicit gradient algorithms is nearly invariant across $K$ values. Such a memory-saving property is important when meta-learning is employed with models of growing degrees of freedom. Furthermore, one may also notice from both figures that MAML and iMAML incur about $50\%$ time/space complexities of ABML and iBaML. This is because non-Bayesian approaches only optimize the mean vector of the Gaussian prior, whose dimension is $d$, while the probabilistic methods cope with both the mean and diagonal covariance matrix of the pdf with corresponding dimension $2d$. This increase in dimensionality doubles the space-time complexity in gradient computations. 

\begin{table}[t]
\centering
\begin{tabular}{lcc}
\toprule
Method & nll & accuracy \\
\midrule
MAML, $K=5$ & $0.967 \pm 0.017$ & $63.1 \pm 0.92\%$ \\
ABML, $K=5$		& $0.957 \pm 0.016$ & $62.8 \pm 0.74 \%$ \\
iBaML, $K=5$	    & $0.965 \pm 0.018$ & $63.2 \pm 0.74 \%$ \\
iBaML, $K=10$	& $0.947 \pm 0.017$ & $64.0 \pm 0.75 \%$ \\
iBaML, $K=15$	& $0.943 \pm 0.017$ & $64.0 \pm 0.74 \%$ \\
\bottomrule
\end{tabular}
\caption{Test negative log-likelihood (nll) and accuracy comparison on $5$-class $5$-shot \textit{mini}ImageNet dataset. The $\pm$ sign indicates the $95\%$ confidence interval.}
\label{tab:nll-acc}
\vspace{-10pt}
\end{table}

Next, we demonstrate the effectiveness of iBaML in reducing the Bayesian meta-learning loss. The test is conducted on the $5$-class $5$-shot \textit{mini}ImageNet. The model is a standard $4$-layer $32$-channel convolutional neural network, and the chosen baseline algorithms are MAML~\cite{MAML} and ABML~\cite{ABML}; see also the Appendix for alternative setups. Due to the large number of training tasks, it is impractical to compute the exact meta-training loss. As an alternative, we adopt the `test nll' (averaged over $1,000$ test tasks) as our metric, and also report their corresponding accuracy. For fairness, we set $L = 5$ when implementing the implicit gradients so that the time complexity is similar to explicit one with $K = 5$. The results are listed in Table~\ref{tab:nll-acc}. It is observed that both nll and accuracy improve with $K$, implying that the meta-learning loss can be effectively reduced by trading a small error in gradient estimation. 

\begin{figure}[t]
\vspace{-5pt}
	\centering
	\includegraphics[width=.8\columnwidth]{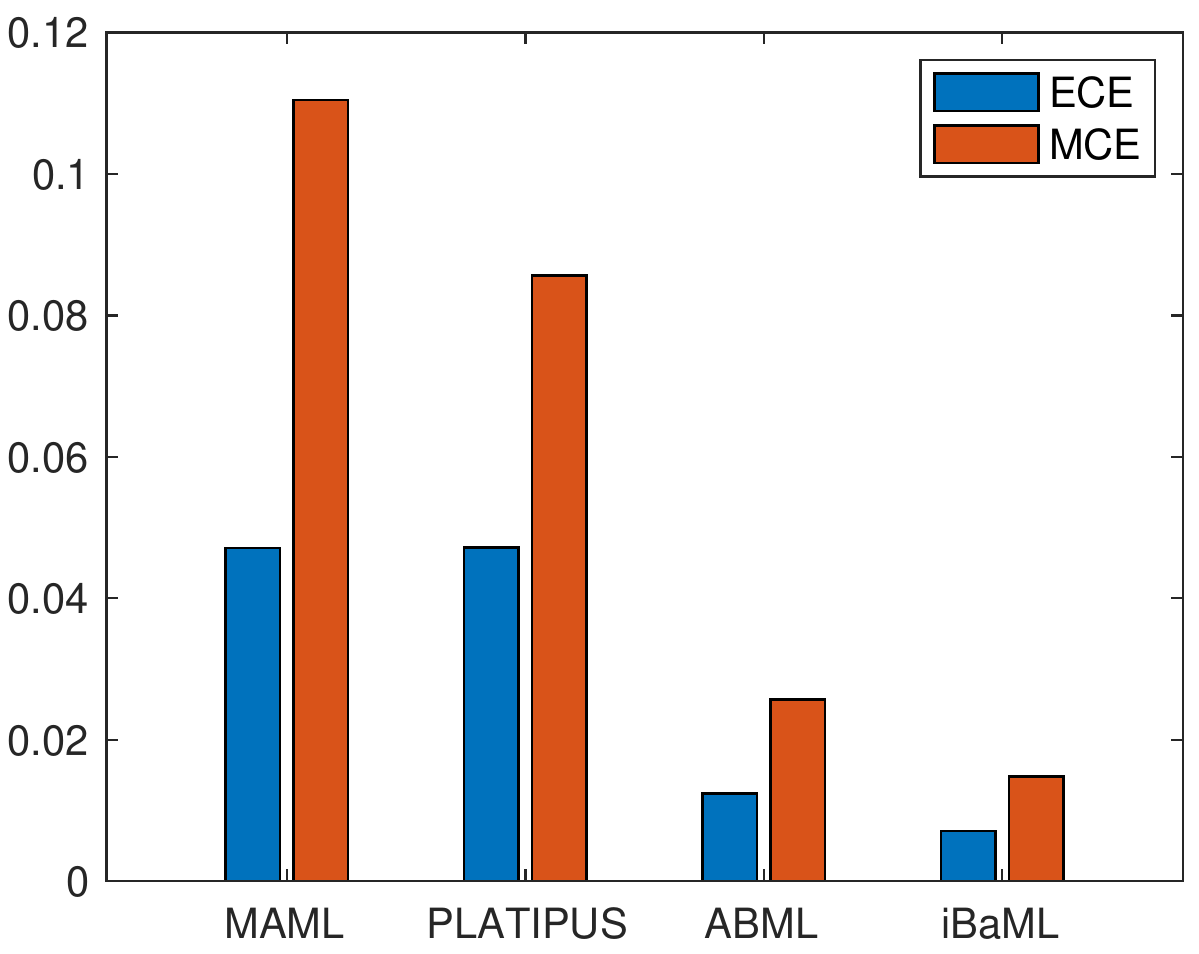}
	\vspace{-5pt}
    \caption{Calibration errors on 5-class 1-shot \textit{mini}ImageNet.}
	\label{fig:ece-mce}
\vspace{-5pt}
\end{figure}


To quantify the uncertainties embedded in state-of-the-art meta-learning methods, Figure~\ref{fig:ece-mce} plots the expected/maximum calibration errors (ECE/MCE)~\cite{MCE-ECE}. It can be seen that iBaML is once again the most competitive among tested approaches.

\section{Conclusions}
This paper develops a novel so-termed iBaML approach to enhance the scalablity of Bayesian meta-learning. At the core of iBaML is an estimate of meta-gradients using implicit differentiation. Analysis reveals that the estimation error is upper bounded by task-level optimization and CG errors, and these two can be significantly reduced with only a slight increase in time complexity. In addition, the required computational complexity is invariant to the task-level optimization trajectory, what allows iBaML to deal with complicated task-level optimization. Besides analytical performance, extensive numerical tests on synthetic and real datasets are also conducted and demonstrate the appealing merits of iBaML over competing alternatives.

\section*{Acknowledgments}
This work was supported in part by NSF grants 2220292, 2212318, 2126052, and 2128593.

\bibliography{ref}

\begin{thebibliography}{37}
\providecommand{\natexlab}[1]{#1}

\bibitem[{Abbas et~al.(2022)Abbas, Xiao, Chen, Chen, and Chen}]{sharp-MAML}
Abbas, M.; Xiao, Q.; Chen, L.; Chen, P.-Y.; and Chen, T. 2022.
\newblock Sharp-{MAML}: Sharpness-Aware Model-Agnostic Meta Learning.
\newblock In \emph{Proceedings of the 39th International Conference on Machine
  Learning}, volume 162 of \emph{Proceedings of Machine Learning Research},
  10--32. PMLR.

\bibitem[{Bengio, Bengio, and Cloutier(1995)}]{convention-2}
Bengio, S.; Bengio, Y.; and Cloutier, J. 1995.
\newblock On the Search for New Learning Rules for ANNs.
\newblock \emph{Neural Processing Letters}, 2(4): 26--30.

\bibitem[{Bertinetto et~al.(2019)Bertinetto, Henriques, Torr, and
  Vedaldi}]{R2D2}
Bertinetto, L.; Henriques, J.~F.; Torr, P.; and Vedaldi, A. 2019.
\newblock Meta-learning with Differentiable Closed-Form Solvers.
\newblock In \emph{Proceedings of International Conference on Learning
  Representations}.

\bibitem[{Botev, Ritter, and Barber(2017)}]{GGN}
Botev, A.; Ritter, H.; and Barber, D. 2017.
\newblock Practical {G}auss-{N}ewton Optimisation for Deep Learning.
\newblock In \emph{Proceedings of the 34th International Conference on Machine
  Learning}, volume~70 of \emph{Proceedings of Machine Learning Research},
  557--565. PMLR.

\bibitem[{Chen and Chen(2022)}]{is-Bayesian-better}
Chen, L.; and Chen, T. 2022.
\newblock Is Bayesian Model-Agnostic Meta Learning Better than Model-Agnostic
  Meta Learning, Provably?
\newblock In \emph{Proceedings of The 25th International Conference on
  Artificial Intelligence and Statistics}, volume 151 of \emph{Proceedings of
  Machine Learning Research}, 1733--1774. PMLR.

\bibitem[{Fallah, Mokhtari, and Ozdaglar(2020)}]{convergence-MAML}
Fallah, A.; Mokhtari, A.; and Ozdaglar, A. 2020.
\newblock On the Convergence Theory of Gradient-Based Model-Agnostic
  Meta-Learning Algorithms.
\newblock In \emph{Proceedings of the Twenty Third International Conference on
  Artificial Intelligence and Statistics}, volume 108, 1082--1092. PMLR.

\bibitem[{Finn, Abbeel, and Levine(2017)}]{MAML}
Finn, C.; Abbeel, P.; and Levine, S. 2017.
\newblock Model-Agnostic Meta-Learning for Fast Adaptation of Deep Networks.
\newblock In \emph{Proceedings of the 34th International Conference on Machine
  Learning}, volume~70, 1126--1135. PMLR.

\bibitem[{Finn, Xu, and Levine(2018)}]{PLATIPUS}
Finn, C.; Xu, K.; and Levine, S. 2018.
\newblock Probabilistic Model-Agnostic Meta-Learning.
\newblock In \emph{Advances in Neural Information Processing Systems},
  volume~31. Curran Associates, Inc.

\bibitem[{Flennerhag et~al.(2020)Flennerhag, Rusu, Pascanu, Visin, Yin, and
  Hadsell}]{WarpGrad}
Flennerhag, S.; Rusu, A.~A.; Pascanu, R.; Visin, F.; Yin, H.; and Hadsell, R.
  2020.
\newblock Meta-Learning with Warped Gradient Descent.
\newblock In \emph{Proceedings of International Conference on Learning
  Representations}.

\bibitem[{Franceschi et~al.(2018)Franceschi, Frasconi, Salzo, Grazzi, and
  Pontil}]{bilevel-programming}
Franceschi, L.; Frasconi, P.; Salzo, S.; Grazzi, R.; and Pontil, M. 2018.
\newblock Bilevel Programming for Hyperparameter Optimization and
  Meta-Learning.
\newblock In \emph{Proceedings of the 35th International Conference on Machine
  Learning}, volume~80, 1568--1577. PMLR.

\bibitem[{Grant et~al.(2018)Grant, Finn, Levine, Darrell, and
  Griffiths}]{LLAMA}
Grant, E.; Finn, C.; Levine, S.; Darrell, T.; and Griffiths, T. 2018.
\newblock Recasting Gradient-Based Meta-Learning as Hierarchical Bayes.
\newblock In \emph{Proceedings of International Conference on Learning
  Representations}.

\bibitem[{Griewank(1993)}]{HVP-complexity}
Griewank, A. 1993.
\newblock Some bounds on the complexity of gradients, Jacobians, and Hessians.
\newblock In \emph{Complexity in numerical optimization}, 128--162. World
  Scientific.

\bibitem[{Hansen and Wang(2021)}]{app-robot}
Hansen, N.; and Wang, X. 2021.
\newblock Generalization in Reinforcement Learning by Soft Data Augmentation.
\newblock In \emph{2021 IEEE International Conference on Robotics and
  Automation (ICRA)}, 13611--13617.

\bibitem[{Kingma and Ba(2015)}]{Adam}
Kingma, D.~P.; and Ba, J. 2015.
\newblock Adam: A Method for Stochastic Optimization.
\newblock In \emph{Proceedings of International Conference on Learning
  Representations}.

\bibitem[{Lee et~al.(2019)Lee, Maji, Ravichandran, and Soatto}]{MetaOptNet}
Lee, K.; Maji, S.; Ravichandran, A.; and Soatto, S. 2019.
\newblock Meta-Learning With Differentiable Convex Optimization.
\newblock In \emph{Proceedings of the IEEE/CVF Conference on Computer Vision
  and Pattern Recognition (CVPR)}.

\bibitem[{Li et~al.(2017)Li, Zhou, Chen, and Li}]{metaSGD}
Li, Z.; Zhou, F.; Chen, F.; and Li, H. 2017.
\newblock Meta-sgd: Learning to learn quickly for few-shot learning.
\newblock \emph{arXiv preprint arXiv:1707.09835}.

\bibitem[{Martens and Grosse(2015)}]{K-FAC}
Martens, J.; and Grosse, R. 2015.
\newblock Optimizing Neural Networks with Kronecker-factored Approximate
  Curvature.
\newblock In \emph{Proceedings of the 32nd International Conference on Machine
  Learning}, volume~37 of \emph{Proceedings of Machine Learning Research},
  2408--2417. Lille, France: PMLR.

\bibitem[{Miao, Metze, and Rawat(2013)}]{app-speech}
Miao, Y.; Metze, F.; and Rawat, S. 2013.
\newblock Deep maxout networks for low-resource speech recognition.
\newblock In \emph{2013 IEEE Workshop on Automatic Speech Recognition and
  Understanding}, 398--403. IEEE.

\bibitem[{Mishra et~al.(2018)Mishra, Rohaninejad, Chen, and
  Abbeel}]{neural-attentive}
Mishra, N.; Rohaninejad, M.; Chen, X.; and Abbeel, P. 2018.
\newblock A Simple Neural Attentive Meta-Learner.
\newblock In \emph{International Conference on Learning Representations}.

\bibitem[{Naeini, Cooper, and Hauskrecht(2015)}]{MCE-ECE}
Naeini, M.~P.; Cooper, G.; and Hauskrecht, M. 2015.
\newblock Obtaining well calibrated probabilities using bayesian binning.
\newblock In \emph{Proceedings of the Twenty Ninth International Conference on
  Artificial Intelligence and Statistics}, 2901--2907. PMLR.

\bibitem[{Nguyen, Do, and Carneiro(2020)}]{VAMPIRE}
Nguyen, C.; Do, T.-T.; and Carneiro, G. 2020.
\newblock Uncertainty in Model-Agnostic Meta-Learning using Variational
  Inference.
\newblock In \emph{Proceedings of the IEEE/CVF Winter Conference on
  Applications of Computer Vision (WACV)}.

\bibitem[{Nichol, Achiam, and Schulman(2018)}]{Reptile}
Nichol, A.; Achiam, J.; and Schulman, J. 2018.
\newblock On First-Order Meta-Learning Algorithms.
\newblock \emph{arXiv preprint arXiv:1803.02999}.

\bibitem[{Paszke et~al.(2019)Paszke, Gross, Massa, Lerer, Bradbury, Chanan,
  Killeen, Lin, Gimelshein, Antiga, Desmaison, Kopf, Yang, DeVito, Raison,
  Tejani, Chilamkurthy, Steiner, Fang, Bai, and Chintala}]{pytorch}
Paszke, A.; Gross, S.; Massa, F.; Lerer, A.; Bradbury, J.; Chanan, G.; Killeen,
  T.; Lin, Z.; Gimelshein, N.; Antiga, L.; Desmaison, A.; Kopf, A.; Yang, E.;
  DeVito, Z.; Raison, M.; Tejani, A.; Chilamkurthy, S.; Steiner, B.; Fang, L.;
  Bai, J.; and Chintala, S. 2019.
\newblock PyTorch: An Imperative Style, High-Performance Deep Learning Library.
\newblock In \emph{Advances in Neural Information Processing Systems},
  volume~32. Curran Associates, Inc.

\bibitem[{Rajeswaran et~al.(2019)Rajeswaran, Finn, Kakade, and Levine}]{iMAML}
Rajeswaran, A.; Finn, C.; Kakade, S.~M.; and Levine, S. 2019.
\newblock Meta-Learning with Implicit Gradients.
\newblock In \emph{Advances in Neural Information Processing Systems},
  volume~32. Curran Associates, Inc.

\bibitem[{Ravi and Beatson(2019)}]{ABML}
Ravi, S.; and Beatson, A. 2019.
\newblock Amortized Bayesian Meta-Learning.
\newblock In \emph{Proceedings of International Conference on Learning
  Representations}.

\bibitem[{Ravi and Larochelle(2017)}]{opt-as}
Ravi, S.; and Larochelle, H. 2017.
\newblock Optimization as a Model for Few-Shot Learning.
\newblock In \emph{Proceedings of International Conference on Learning
  Representations}.

\bibitem[{Santoro et~al.(2016)Santoro, Bartunov, Botvinick, Wierstra, and
  Lillicrap}]{memory-aug}
Santoro, A.; Bartunov, S.; Botvinick, M.; Wierstra, D.; and Lillicrap, T. 2016.
\newblock Meta-Learning with Memory-Augmented Neural Networks.
\newblock In \emph{Proceedings of the 33rd International Conference on Machine
  Learning}, volume~48, 1842--1850. New York, New York, USA: PMLR.

\bibitem[{Schmidhuber(1993)}]{convention-1}
Schmidhuber, J. 1993.
\newblock A Neural Network that Embeds its Own Meta-Levels.
\newblock In \emph{IEEE International Conference on Neural Networks}, 407--412
  vol.1.

\bibitem[{Schmidhuber, Zhao, and Wiering(1996)}]{convention-3}
Schmidhuber, J.; Zhao, J.; and Wiering, M. 1996.
\newblock Simple Principles of Metalearning.
\newblock \emph{Technical report IDSIA}, 69: 1--23.

\bibitem[{Thrun(1998)}]{lifelong-learning}
Thrun, S. 1998.
\newblock \emph{Lifelong Learning Algorithms}, 181--209.
\newblock Boston, MA: Springer US.
\newblock ISBN 978-1-4615-5529-2.

\bibitem[{Thrun and Pratt(2012)}]{learning-to-learn}
Thrun, S.; and Pratt, L. 2012.
\newblock \emph{Learning to Learn}.
\newblock Springer Science \& Business Media.

\bibitem[{Van~der Sluis and van~der Vorst(1986)}]{CG-convergence-1}
Van~der Sluis, A.; and van~der Vorst, H.~A. 1986.
\newblock The rate of convergence of conjugate gradients.
\newblock \emph{Numerische Mathematik}, 48(5): 543--560.

\bibitem[{Vinyals et~al.(2016)Vinyals, Blundell, Lillicrap, kavukcuoglu, and
  Wierstra}]{MatchingNets}
Vinyals, O.; Blundell, C.; Lillicrap, T.; kavukcuoglu, k.; and Wierstra, D.
  2016.
\newblock Matching Networks for One Shot Learning.
\newblock In \emph{Advances in Neural Information Processing Systems},
  volume~29. Curran Associates, Inc.

\bibitem[{Wang, Sun, and Li(2020)}]{global-convergence}
Wang, H.; Sun, R.; and Li, B. 2020.
\newblock Global Convergence and Generalization Bound of Gradient-Based
  Meta-Learning with Deep Neural Nets.
\newblock \emph{arXiv preprint arXiv:2006.14606}.

\bibitem[{Winther(1980)}]{CG-convergence-2}
Winther, R. 1980.
\newblock Some Superlinear Convergence Results for the Conjugate Gradient
  Method.
\newblock \emph{SIAM Journal on Numerical Analysis}, 17(1): 14--17.

\bibitem[{yang et~al.(2016)yang, Sun, Li, and Xu}]{app-med}
yang, y.; Sun, J.; Li, H.; and Xu, Z. 2016.
\newblock Deep ADMM-Net for Compressive Sensing MRI.
\newblock In \emph{Advances in Neural Information Processing Systems},
  volume~29. Curran Associates, Inc.

\bibitem[{Yoon et~al.(2018)Yoon, Kim, Dia, Kim, Bengio, and Ahn}]{BMAML}
Yoon, J.; Kim, T.; Dia, O.; Kim, S.; Bengio, Y.; and Ahn, S. 2018.
\newblock Bayesian Model-Agnostic Meta-Learning.
\newblock In \emph{Advances in Neural Information Processing Systems},
  volume~31. Curran Associates, Inc.

\end{thebibliography}

\clearpage
\onecolumn
\setcounter{lemma}{0}
\setcounter{theorem}{0}

\section*{Appendix}
\subsection*{A.1 Proof of Lemma~\ref{lemma:implicit_Jacobi}}
\begin{lemma}[Restated]
Consider the Bayesian meta-learning problem~\eqref{eq:Bayesian-meta-VI}. Let $\bar{\mathbf{v}}_t := [\bar{\mathbf{m}}_t^\top, \bar{\mathbf{d}}_t^\top]^\top$ be a local minimum of the task-level KL-divergence generated by $\bar{\mathcal{A}}_t (\boldsymbol{\theta})$; and, $\mathcal{L}_t^{\mathrm{tr}}(\mathbf{v}_t) := \mathbb{E}_{q(\boldsymbol{\theta}_t; \mathbf{v}_t)} [ -\log p(\mathbf{y}_t^{\mathrm{tr}} | \boldsymbol{\theta}_t; \mathbf{X}_t^{\mathrm{tr}}) ]$ the expected negative log-likelihood (nll) on $\mathcal{D}_t^{\mathrm{tr}}$. If $\mathbf{H}_t (\bar{\mathbf{v}}_t) := \nabla^2 \mathcal{L}_t^{\mathrm{tr}} (\bar{\mathbf{v}}_t) + 
	\left[ \begin{matrix}
		\mathbf{D}^{-1} & \mathbf{0}_d \\
		\mathbf{0}_d & \frac{1}{2} \big( \mathbf{D}^{-1} + 2 \diag \big( \nabla_{\bar{\mathbf{d}}_t} \mathcal{L}_t^{\mathrm{tr}} (\bar{\mathbf{v}}_t) \big) \big)^2
	\end{matrix} \right]$ is invertible, it then holds for $t \in \{ 1, \ldots, T \}$ that
\begin{equation}
	\nabla \bar{\mathcal{A}}_t (\boldsymbol{\theta}) = \left[ \begin{matrix}
		\mathbf{D}^{-1} & \mathbf{0}_d \\
		-\diag(\nabla_{\bar{\mathbf{m}}_t} \mathcal{L}_t^{\mathrm{tr}}(\bar{\mathbf{v}}_t) ) \mathbf{D}^{-1} & \frac{1}{2}\mathbf{D}^{-2}
	\end{matrix} \right]
	\mathbf{H}_t^{-1} (\bar{\mathbf{v}}_t).
\end{equation}
\end{lemma}

\begin{proof}
We first write out the evidence lower bound (ELBO) of the VI in~\eqref{eq:Bayesian-global-min}. 
\begin{align}
	&\KL \big( q(\boldsymbol{\theta}_t ; \mathbf{v}_t) \big\| p(\boldsymbol{\theta}_t | \mathbf{y}_t^{\mathrm{tr}}; \mathbf{X}_t^{\mathrm{tr}}, \boldsymbol{\theta} ) \big) \nonumber 
	= \int q(\boldsymbol{\theta}_t ; \mathbf{v}_t) \log \frac{q(\boldsymbol{\theta}_t ; \mathbf{v}_t)}{p(\boldsymbol{\theta}_t | \mathbf{y}_t^{\mathrm{tr}}; \mathbf{X}_t^{\mathrm{tr}}, \boldsymbol{\theta} )} \nonumber 
	= \int q(\boldsymbol{\theta}_t ; \mathbf{v}_t) \log \frac{q(\boldsymbol{\theta}_t ; \mathbf{v}_t) p(\mathbf{y}_t^{\mathrm{tr}}; \mathbf{X}_t^{\mathrm{tr}}, \boldsymbol{\theta})} {p(\mathbf{y}_t^{\mathrm{tr}}, \boldsymbol{\theta}_t; \mathbf{X}_t^{\mathrm{tr}}, \boldsymbol{\theta} )} \nonumber \\
	&= \int q(\boldsymbol{\theta}_t ; \mathbf{v}_t) \log \frac{q(\boldsymbol{\theta}_t ; \mathbf{v}_t) p(\mathbf{y}_t^{\mathrm{tr}}; \mathbf{X}_t^{\mathrm{tr}}, \boldsymbol{\theta})} {p(\mathbf{y}_t^{\mathrm{tr}} | \boldsymbol{\theta}_t; \mathbf{X}_t^{\mathrm{tr}} ) p(\boldsymbol{\theta}_t; \boldsymbol{\theta})} \nonumber 
	= \mathbb{E}_{q(\boldsymbol{\theta}_t ; \mathbf{v}_t)} [ -\log p(\mathbf{y}_t^{\mathrm{tr}} | \boldsymbol{\theta}_t; \mathbf{X}_t^{\mathrm{tr}}) ] + \mathbb{E}_{q(\boldsymbol{\theta}_t ; \mathbf{v}_t)} \Big[ \log \frac{q(\boldsymbol{\theta}_t ; \mathbf{v}_t)} {p(\boldsymbol{\theta}_t; \boldsymbol{\theta})} \Big] \\
	&+ \mathbb{E}_{q(\boldsymbol{\theta}_t ; \mathbf{v}_t)} [\log p(\mathbf{y}_t^{\mathrm{tr}}; \mathbf{X}_t^{\mathrm{tr}}, \boldsymbol{\theta} )] \nonumber 
	= \mathcal{L}_t^{\mathrm{tr}} (\mathbf{v}_t) + \KL \big( q(\boldsymbol{\theta}_t ; \mathbf{v}_t) \big\| p(\boldsymbol{\theta}_t; \boldsymbol{\theta} ) \big) + \log p(\mathbf{y}_t^{\mathrm{tr}}; \mathbf{X}_t^{\mathrm{tr}}, \boldsymbol{\theta}) \nonumber = -\mathrm{ELBO} + \log p(\mathbf{y}_t^{\mathrm{tr}}; \mathbf{X}_t^{\mathrm{tr}}, \boldsymbol{\theta})
\end{align} 
where $\mathrm{ELBO} := -\mathcal{L}_t^{\mathrm{tr}} (\mathbf{v}_t) - \KL \big( q(\boldsymbol{\theta}_t ; \mathbf{v}_t) \big\| p(\boldsymbol{\theta}_t; \boldsymbol{\theta} ) \big)$. Minimizing the KL divergence 
amounts to maximizing the ELBO. 

From the definitions $\boldsymbol{\theta} := [\mathbf{m}^\top, \mathbf{d}^\top]^\top$ and $\bar{\mathbf{v}}_t := \bar{\mathcal{A}}_t (\boldsymbol{\theta}) = [\bar{\mathbf{m}}_t^\top, \bar{\mathbf{d}}_t^\top]^\top$, we can write the desired gradient as a block matrix
\begin{equation}
	\nabla \bar{\mathcal{A}}_t (\boldsymbol{\theta}) = 
	\left[ \begin{matrix}
		\nabla_{\mathbf{m}} \bar{\mathbf{m}}_t & \nabla_{\mathbf{m}} \bar{\mathbf{d}}_t \\
		\nabla_{\mathbf{d}} \bar{\mathbf{m}}_t & \nabla_{\mathbf{d}} \bar{\mathbf{d}}_t
	\end{matrix} \right]
\end{equation}
where with a slight abuse in notation $\nabla_{\mathbf{m}} \bar{\mathcal{A}}_t (\boldsymbol{\theta}) = [\nabla_{\mathbf{m}} \bar{\mathbf{m}}_t, \nabla_{\mathbf{m}} \bar{\mathbf{d}}_t]$ and $\nabla_{\mathbf{d}} \bar{\mathcal{A}}_t (\boldsymbol{\theta}) = [\nabla_{\mathbf{d}} \bar{\mathbf{m}}_t, \nabla_{\mathbf{d}} \bar{\mathbf{d}}_t]$ denote partial gradients. The next step is to express $\nabla_{\mathbf{m}} \bar{\mathcal{A}}_t (\boldsymbol{\theta})$ as a function of itself to leverage the implicit differentiation. 

Since $\bar{\mathbf{\mathbf{v}}}_t$ is a local minimum of $\KL \big( q(\boldsymbol{\theta}_t ; \mathbf{v}_t) \big\| p(\boldsymbol{\theta}_t ; \mathbf{y}_t^{\mathrm{tr}}; \mathbf{X}_t^{\mathrm{tr}}, \boldsymbol{\theta} ) \big)$, it maximizes the ELBO. The first-order necessary condition for optimality thus yields
\begin{equation}
\label{eq:first-order-cond}
	-\nabla \mathcal{L}_t^{\mathrm{tr}} (\bar{\mathbf{v}}_t) - \nabla_{\bar{\mathbf{v}}_t} \KL \big( q(\boldsymbol{\theta}_t; \bar{\mathbf{v}}_t) \big\| p(\boldsymbol{\theta}_t; \boldsymbol{\theta}) \big) = \mathbf{0}.
\end{equation}
Upon defining $\bar{\mathbf{D}}_t := \diag(\bar{\mathbf{d}}_t)$, the KL-divergence of Gaussian distributions can be written as 
\begin{align}
\label{eq:KL-Gaussian}
	\KL \big( q(\boldsymbol{\theta}_t; \bar{\mathbf{v}}_t) \big\| p(\boldsymbol{\theta}_t; \boldsymbol{\theta}) \big) 
	& = \frac{1}{2} \Big( \tr( \mathbf{D}^{-1} \bar{\mathbf{D}}_t ) - n + (\mathbf{m} - \bar{\mathbf{m}}_t)^\top \mathbf{D}^{-1} (\mathbf{m} - \bar{\mathbf{m}}_t) + \log \frac{| \mathbf{D}|}{| \bar{\mathbf{D}}_t |} \Big) \nonumber \\
	&= \frac{1}{2} \sum_{i=1}^d  \Big(\frac{[\bar{\mathbf{d}}_t]_i}{[\mathbf{d}]_i} - 1 + \frac{([\mathbf{m}]_i - [\bar{\mathbf{m}}_t]_i)^2}{[\mathbf{d}]_i} + \log [\mathbf{d}]_i - \log [\bar{\mathbf{d}}_t]_i \Big),
\end{align}
and after plugging~\eqref{eq:KL-Gaussian} into~\eqref{eq:first-order-cond} and rearranging terms, we arrive at 
\begin{equation}
\label{eq:post-mean}
	\bar{\mathbf{m}}_t = \mathbf{m} - \mathbf{D} \nabla_{\bar{\mathbf{m}}_t} \mathcal{L}_t^{\mathrm{tr}} (\bar{\mathbf{v}}_t)
\end{equation} 
and
\begin{equation}
\label{eq:post-var}
	\bar{\mathbf{d}}_t = \Big( \mathbf{d}^{-1} + 2\nabla_{\bar{\mathbf{d}}_t} \mathcal{L}_t^{\mathrm{tr}} (\bar{\mathbf{v}}_t) \Big)^{-1}
\end{equation}
where we used $\mathbf{v}^{-1}$ to represent the element-wise inverse of a general vector $\mathbf{v}$. 

Then, taking gradient w.r.t. $\boldsymbol{\theta} = [ \mathbf{m}^\top, \mathbf{d}^\top ]^\top$ on both sides of~\eqref{eq:post-mean}, and employing the chain rule results in
\begin{equation}
\label{eq:implicit-relation-1}
	\nabla_{\mathbf{m}} \bar{\mathbf{m}}_t = \mathbf{I}_d - \big( \nabla_{\mathbf{m}} \bar{\mathbf{m}}_t \nabla_{\bar{\mathbf{m}}_t}^2 \mathcal{L}_t^{\mathrm{tr}} (\bar{\mathbf{v}}_t) + \nabla_{\mathbf{m}} \bar{\mathbf{d}}_t \nabla_{\bar{\mathbf{d}}_t} \nabla_{\bar{\mathbf{m}}_t} \mathcal{L}_t^{\mathrm{tr}} (\bar{\mathbf{v}}_t) \big) \mathbf{D}
\end{equation}
and 
\begin{equation}
	\nabla_{\mathbf{d}} \bar{\mathbf{m}}_t = -\diag \big( \nabla_{\bar{\mathbf{m}}_t} \mathcal{L}_t^{\mathrm{tr}} (\bar{\mathbf{v}}_t) \big) - \big( \nabla_{\mathbf{d}} \bar{\mathbf{m}}_t \nabla_{\bar{\mathbf{m}}_t}^2 \mathcal{L}_t^{\mathrm{tr}} (\bar{\mathbf{v}}_t) + \nabla_{\mathbf{d}} \bar{\mathbf{d}}_t \nabla_{\bar{\mathbf{d}}_t} \nabla_{\bar{\mathbf{m}}_t} \mathcal{L}_t^{\mathrm{tr}} (\bar{\mathbf{v}}_t) \big) \mathbf{D}.
\end{equation}
Applying the same operation to~\eqref{eq:post-var}, yields 
\begin{equation}
	\nabla_{\mathbf{m}} \bar{\mathbf{d}}_t = -2 \big( \nabla_{\mathbf{m}} \bar{\mathbf{d}}_t \nabla_{\bar{\mathbf{d}}_t}^2 \mathcal{L}_t^{\mathrm{tr}} (\bar{\mathbf{v}}_t) + \nabla_{\mathbf{m}} \bar{\mathbf{m}}_t \nabla_{\bar{\mathbf{m}}_t} \nabla_{\bar{\mathbf{d}}_t} \mathcal{L}_t^{\mathrm{tr}} (\bar{\mathbf{v}}_t) \big) \bar{\mathbf{D}}_t^2
\end{equation}
and 
\begin{equation}
\label{eq:implicit-relation-4}
	\nabla_{\mathbf{d}} \bar{\mathbf{d}}_t = -\big( -\mathbf{D}^{-2} + 2\nabla_{\mathbf{d}} \bar{\mathbf{d}}_t \nabla_{\bar{\mathbf{d}}_t}^2 \mathcal{L}_t^{\mathrm{tr}} (\bar{\mathbf{v}}_t) + 2\nabla_{\mathbf{d}} \bar{\mathbf{m}}_t \nabla_{\bar{\mathbf{m}}_t} \nabla_{\bar{\mathbf{d}}_t} \mathcal{L}_t^{\mathrm{tr}} (\bar{\mathbf{v}}_t) \big) \bar{\mathbf{D}}_t^2.
\end{equation}
So far, we have written the four blocks of $\nabla \bar{\mathcal{A}}_t (\boldsymbol{\theta})$ as a function of themselves through implicit differentiation. Hence, the last step is to solve for these four blocks from the linear equations~\eqref{eq:implicit-relation-1}-\eqref{eq:implicit-relation-4}. 

Directly solving this linear system of equations will produce complicated results. The trick here is to reformulate them into a compact matrix form:
\begin{align}
	&\nabla \bar{\mathcal{A}}_t (\boldsymbol{\theta}) \nonumber \\
	&= \left(
	\left[ \begin{matrix}
		\mathbf{I}_d & \mathbf{0}_d \\
		-\diag \big( \nabla_{\bar{\mathbf{m}}_t} \mathcal{L}_t^{\mathrm{tr}} (\bar{\mathbf{v}}_t) \big) & -\mathbf{D}^{-2}
	\end{matrix} \right]
	- \left[ \begin{matrix}
		\nabla_{\mathbf{m}} \bar{\mathbf{m}}_t & \nabla_{\mathbf{m}} \bar{\mathbf{d}}_t \\
		\nabla_{\mathbf{d}} \bar{\mathbf{m}}_t & \nabla_{\mathbf{d}} \bar{\mathbf{d}}_t
	\end{matrix} \right]
	\left[ \begin{matrix}
		\nabla_{\bar{\mathbf{m}}_t}^2 \mathcal{L}_t^{\mathrm{tr}} (\bar{\mathbf{v}}_t)  & \nabla_{\bar{\mathbf{m}}_t} \nabla_{\bar{\mathbf{d}}_t} \mathcal{L}_t^{\mathrm{tr}} (\bar{\mathbf{v}}_t) \\
		\nabla_{\bar{\mathbf{d}}_t} \nabla_{\bar{\mathbf{m}}_t} \mathcal{L}_t^{\mathrm{tr}} (\bar{\mathbf{v}}_t) & \nabla_{\bar{\mathbf{d}}_t}^2 \mathcal{L}_t^{\mathrm{tr}} (\bar{\mathbf{v}}_t)
	\end{matrix} \right]
	\left[ \begin{matrix}
		\mathbf{D} & \mathbf{0}_d \\
		\mathbf{0}_d & -2\mathbf{I}_d
	\end{matrix} \right]
	\right) \nonumber \\
	&~~~~~ \times \left[ \begin{matrix}
		\mathbf{I}_d & \mathbf{0}_d \\
		\mathbf{0}_d & -\bar{\mathbf{D}}_t^2
	\end{matrix} \right] \nonumber \\
	&= \left(
	\left[ \begin{matrix}
		\mathbf{I}_d & \mathbf{0}_d \\
		-\diag \big( \nabla_{\bar{\mathbf{m}}_t} \mathcal{L}_t^{\mathrm{tr}} (\bar{\mathbf{v}}_t) \big) & -\mathbf{D}^{-2}
	\end{matrix} \right]
	- \nabla \bar{\mathcal{A}}_t (\boldsymbol{\theta})
	\nabla^2 \mathcal{L}_t^{\mathrm{tr}} (\bar{\mathbf{v}}_t) 
	\left[ \begin{matrix}
		\mathbf{D} & \mathbf{0}_d \\
		\mathbf{0}_d & -2\mathbf{I}_d
	\end{matrix} \right]
	\right)
	\left[ \begin{matrix}
		\mathbf{I}_d & \mathbf{0}_d \\
		\mathbf{0}_d & -\bar{\mathbf{D}}_t^2
	\end{matrix} \right] \nonumber \\
	&= \left(
	\left[ \begin{matrix}
		\mathbf{I}_d & \mathbf{0}_d \\
		-\diag \big( \nabla_{\bar{\mathbf{m}}_t} \mathcal{L}_t^{\mathrm{tr}} (\bar{\mathbf{v}}_t) \big) & \mathbf{D}^{-2}
	\end{matrix} \right]
	- \nabla \bar{\mathcal{A}}_t (\boldsymbol{\theta})
	\nabla^2 \mathcal{L}_t^{\mathrm{tr}} (\bar{\mathbf{v}}_t) 
	\left[ \begin{matrix}
		\mathbf{D} & \mathbf{0}_d \\
		\mathbf{0}_d & 2\mathbf{I}_d
	\end{matrix} \right]
	\right)
	\left[ \begin{matrix}
		\mathbf{I}_d & \mathbf{0}_d \\
		\mathbf{0}_d & \bar{\mathbf{D}}_t^2
	\end{matrix} \right].
\end{align}
Now, the matrix equation can be readily solved to obtain
\begin{align}
\label{eq:solve-compact-eq}
	\nabla \bar{\mathcal{A}}_t (\boldsymbol{\theta})
	&= 
	\left[ \begin{matrix}
		\mathbf{I}_d & \mathbf{0}_d \\
		-\diag \big( \nabla_{\bar{\mathbf{m}}_t} \mathcal{L}_t^{\mathrm{tr}}(\bar{\mathbf{v}}_t) \big) & \mathbf{D}^{-2}
	\end{matrix} \right]
	\left( 
	\nabla^2 \mathcal{L}_t^{\mathrm{tr}} (\bar{\mathbf{v}}_t) 
	\left[ \begin{matrix}
		\mathbf{D} & \mathbf{0}_d \\
		\mathbf{0}_d & 2\mathbf{I}_d
	\end{matrix} \right]
	+ \left[ \begin{matrix}
		\mathbf{I}_d & \mathbf{0}_d \\
		\mathbf{0}_d & \bar{\mathbf{D}}_t^{-2}
	\end{matrix} \right]
	\right)^{-1} \nonumber \\
	&= \left[ \begin{matrix}
		\mathbf{I}_d & \mathbf{0}_d \\
		-\diag \big( \nabla_{\bar{\mathbf{m}}_t} \mathcal{L}_t^{\mathrm{tr}}(\bar{\mathbf{v}}_t) \big) & \mathbf{D}^{-2}
	\end{matrix} \right]
	\left( \Big(
	\nabla^2 \mathcal{L}_t^{\mathrm{tr}} (\bar{\mathbf{v}}_t)
	+ \left[ \begin{matrix}
		\mathbf{D}^{-1} & \mathbf{0}_d \\
		\mathbf{0}_d & \frac{1}{2} \bar{\mathbf{D}}_t^{-2}
	\end{matrix} \right] \Big)
	\left[ \begin{matrix}
		\mathbf{D} & \mathbf{0}_d \\
		\mathbf{0}_d & 2\mathbf{I}_d
	\end{matrix} \right]
	\right)^{-1} \nonumber \\
	&= \left[ \begin{matrix}
		\mathbf{D}^{-1} & \mathbf{0}_d \\
		-\diag \big( \nabla_{\bar{\mathbf{m}}_t} \mathcal{L}_t^{\mathrm{tr}}(\bar{\mathbf{v}}_t) \big) \mathbf{D}^{-1} & \frac{1}{2}\mathbf{D}^{-2}
	\end{matrix} \right]
	\left(
	\nabla^2 \mathcal{L}_t^{\mathrm{tr}} (\bar{\mathbf{v}}_t) 
	+ \left[ \begin{matrix}
		\mathbf{D}^{-1} & \mathbf{0}_d \\
		\mathbf{0}_d & \frac{1}{2} \bar{\mathbf{D}}_t^{-2}
	\end{matrix} \right]
	\right)^{-1} \nonumber \\
	&= \left[ \begin{matrix}
		\mathbf{D}^{-1} & \mathbf{0}_d \\
		-\diag \big( \nabla_{\bar{\mathbf{m}}_t} \mathcal{L}_t^{\mathrm{tr}}(\bar{\mathbf{v}}_t) \big) \mathbf{D}^{-1} & \frac{1}{2}\mathbf{D}^{-2}
	\end{matrix} \right]
	\left(
	\nabla^2 \mathcal{L}_t^{\mathrm{tr}} (\bar{\mathbf{v}}_t) 
	+ \left[ \begin{matrix}
		\mathbf{D}^{-1} & \mathbf{0}_d \\
		\mathbf{0}_d & \frac{1}{2} \big( \mathbf{D}^{-1} + 2 \diag \big( \nabla_{\bar{\mathbf{d}}_t} \mathcal{L}_t^{\mathrm{tr}} (\bar{\mathbf{v}}_t) \big) \big)^2
	\end{matrix} \right]
	\right)^{-1} \nonumber \\
	&= \left[ \begin{matrix}
		\mathbf{D}^{-1} & \mathbf{0}_d \\
		-\diag(\nabla_{\bar{\mathbf{m}}_t} \mathcal{L}_t^{\mathrm{tr}}(\bar{\mathbf{v}}_t) ) \mathbf{D}^{-1} & \frac{1}{2}\mathbf{D}^{-2}
	\end{matrix} \right]
	\mathbf{H}_t^{-1} (\bar{\mathbf{v}}_t)
\end{align}
where the fourth equality comes from~\eqref{eq:post-var}.

\end{proof}

\subsection*{A.2 Proof of Theorem~\ref{theor:explicit}}
\begin{theorem}[Explicit meta-gradient error bound, restated]
Consider the Bayesian meta-learning problem in~\eqref{eq:Bayesian-subopt}. Let $\epsilon_t := \| \hat{\mathbf{v}}_t - \bar{\mathbf{v}}_t \|_2$ be the task-level optimization error, and $\delta_t := \| \nabla \hat{\mathcal{A}}_t (\boldsymbol{\theta}) - \mathbf{G}_t (\hat{\mathbf{v}}_t) \mathbf{H}_t^{-1} (\hat{\mathbf{v}}_t) \|_2$ the error of the Jacobian. Upon defining $\rho_t := \max \big\{ \| \nabla_{\bar{\mathbf{v}}_t} \mathcal{L}_t^{\mathrm{tr}} (\bar{\mathbf{v}}_t) \|_{\infty}, \| \nabla_{\hat{\mathbf{v}}_t} \mathcal{L}_t^{\mathrm{tr}} (\hat{\mathbf{v}}_t) \|_{\infty} \big\}$, and with Assumptions~\ref{as:local-min}-\ref{as:bounded-var} in effect, it holds for $t \in \{ 1,\ldots, T \}$ that
\begin{equation}
	\big\| \nabla_{\boldsymbol{\theta}} \mathcal{L}_t^{\mathrm{val}} \big( \hat{\mathbf{v}}_t (\boldsymbol{\theta}), \boldsymbol{\theta} \big) -\nabla_{\boldsymbol{\theta}} \mathcal{L}_t^{\mathrm{val}} \big( \bar{\mathbf{v}}_t (\boldsymbol{\theta}), \boldsymbol{\theta} \big) \big\|_2 \le F_t \epsilon_t + A_t \delta_t,
\end{equation}
where the scalar $F_t$ depends on $\rho_t$. 
\end{theorem}

\begin{proof}
First, it follows by definition~\eqref{eq:Bayesian-meta-grad} of Bayesian meta-gradient that 
\begin{align}
\label{eq:theor1-target}
	&\big\| 
	\nabla_{\boldsymbol{\theta}} \mathcal{L}_t^{\mathrm{val}} \big( \hat{\mathbf{v}}_t (\boldsymbol{\theta}), \boldsymbol{\theta} \big) 
	- \nabla_{\boldsymbol{\theta}} \mathcal{L}_t^{\mathrm{val}} \big( \bar{\mathbf{v}}_t (\boldsymbol{\theta}), \boldsymbol{\theta} \big) 
	\big\|_2 \nonumber \\
	&\le \big\| \nabla \hat{\mathcal{A}}_t (\boldsymbol{\theta}) \nabla_1 \mathcal{L}_t^{\mathrm{val}}(\hat{\mathbf{v}}_t, \boldsymbol{\theta}) - \nabla \bar{\mathcal{A}}_t (\boldsymbol{\theta}) \nabla_1 \mathcal{L}_t^{\mathrm{val}}(\bar{\mathbf{v}}_t, \boldsymbol{\theta}) \big\|_2
	+ \big\| \nabla_2 \mathcal{L}_t^{\mathrm{val}}(\hat{\mathbf{v}}_t, \boldsymbol{\theta}) - \nabla_2 \mathcal{L}_t^{\mathrm{val}}(\bar{\mathbf{v}}_t, \boldsymbol{\theta}) \big\|_2 \nonumber \\
	&\le \big\| \nabla \hat{\mathcal{A}}_t (\boldsymbol{\theta}) \nabla_1 \mathcal{L}_t^{\mathrm{val}}(\hat{\mathbf{v}}_t, \boldsymbol{\theta}) - \nabla \bar{\mathcal{A}}_t (\boldsymbol{\theta}) \nabla_1 \mathcal{L}_t^{\mathrm{val}}(\bar{\mathbf{v}}_t, \boldsymbol{\theta}) \big\|_2
	+ C_t \epsilon_t \nonumber \\
	&\le \big\| \nabla \hat{\mathcal{A}}_t (\boldsymbol{\theta}) \nabla_1 \mathcal{L}_t^{\mathrm{val}}(\hat{\mathbf{v}}_t, \boldsymbol{\theta}) - \nabla \bar{\mathcal{A}}_t (\boldsymbol{\theta}) \nabla_1 \mathcal{L}_t^{\mathrm{val}}(\hat{\mathbf{v}}_t, \boldsymbol{\theta}) \big\|_2 \nonumber \\
	&~~~~~+ \big\| \nabla \bar{\mathcal{A}}_t (\boldsymbol{\theta}) \nabla_1 \mathcal{L}_t^{\mathrm{val}}(\hat{\mathbf{v}}_t, \boldsymbol{\theta}) - \nabla \bar{\mathcal{A}}_t (\boldsymbol{\theta}) \nabla_1 \mathcal{L}_t^{\mathrm{val}}(\bar{\mathbf{v}}_t, \boldsymbol{\theta}) \big\|_2
	+ C_t \epsilon_t \nonumber \\
	&\le \big\| \nabla \hat{\mathcal{A}}_t (\boldsymbol{\theta}) - \nabla \bar{\mathcal{A}}_t (\boldsymbol{\theta}) \big\|_2 \big\| \nabla_1 \mathcal{L}_t^{\mathrm{val}} (\hat{\mathbf{v}}_t, \boldsymbol{\theta}) \big\|_2 
	+ \big\| \nabla \bar{\mathcal{A}}_t (\boldsymbol{\theta}) \|_2 \big\| \nabla_1 \mathcal{L}_t^{\mathrm{val}}(\hat{\mathbf{v}}_t, \boldsymbol{\theta}) - \nabla_1 \mathcal{L}_t^{\mathrm{val}}(\bar{\mathbf{v}}_t, \boldsymbol{\theta}) \big\|_2
	+ C_t \epsilon_t \nonumber \\
	&\le A_t \big\| \nabla \hat{\mathcal{A}}_t (\boldsymbol{\theta}) - \nabla \bar{\mathcal{A}}_t (\boldsymbol{\theta}) \big\|_2 
	+ B_t \epsilon_t \big\| \nabla \bar{\mathcal{A}}_t (\boldsymbol{\theta}) \big\|_2 + C_t \epsilon_t,
\end{align}
where  Assumption~\ref{as:meta-loss} was used in the second and last inequalities. What remains is to bound $\| \nabla \hat{\mathcal{A}}_t (\boldsymbol{\theta}) - \nabla \bar{\mathcal{A}}_t (\boldsymbol{\theta}) \|_2$ and $\| \nabla \bar{\mathcal{A}}_t (\boldsymbol{\theta}) \|_2$. 

Using Lemma~\ref{lemma:implicit_Jacobi} with Assumption~\ref{as:local-min}, we obtain
\begin{align}
	\nabla \bar{\mathcal{A}}_t (\boldsymbol{\theta}) 
	&= \left[ \begin{matrix}
		\mathbf{D}^{-1} & \mathbf{0}_d \\
		-\diag \big( \nabla_{\bar{\mathbf{m}}_t} \mathcal{L}_t^{\mathrm{tr}} (\bar{\mathbf{v}}_t) \big) \mathbf{D}^{-1} & \frac{1}{2}\mathbf{D}^{-2}
	\end{matrix} \right]
	\mathbf{H}_t^{-1} (\bar{\mathbf{v}}_t) \nonumber \\
	&= \left[ \begin{matrix}
		\mathbf{D} & \mathbf{0}_d \\
		2\mathbf{D}^2 \diag \big( \nabla_{\bar{\mathbf{m}}_t} \mathcal{L}_t^{\mathrm{tr}} (\bar{\mathbf{v}}_t) \big) & 2\mathbf{D}^2
	\end{matrix} \right]^{-1}
	\mathbf{H}_t^{-1} (\bar{\mathbf{v}}_t) \nonumber \\
	&=	\bigg( 
	\nabla^2 \mathcal{L}_t^{\mathrm{tr}} (\bar{\mathbf{v}}_t)
	\left[ \begin{matrix}
		\mathbf{D} & \mathbf{0}_d \\
		2\mathbf{D}^2 \diag \big( \nabla_{\bar{\mathbf{m}}_t} \mathcal{L}_t^{\mathrm{tr}} (\bar{\mathbf{v}}_t) \big) & 2\mathbf{D}^2
	\end{matrix} \right] \nonumber \\
	&~~~~~+ \left[ \begin{matrix}
		\mathbf{D}^{-1} & \mathbf{0}_d \\
		\mathbf{0}_d & \frac{1}{2} \big(\mathbf{D}^{-1} + 2 \diag \big( \nabla_{\bar{\mathbf{d}}_t} \mathcal{L}_t^{\mathrm{tr}} (\bar{\mathbf{v}}_t) \big) \big)^2
	\end{matrix} \right]
	\left[ \begin{matrix}
		\mathbf{D} & \mathbf{0}_d \\
		2\mathbf{D}^2 \diag \big( \nabla_{\bar{\mathbf{m}}_t} \mathcal{L}_t^{\mathrm{tr}} (\bar{\mathbf{v}}_t) \big) & 2\mathbf{D}^2
	\end{matrix} \right]
	\bigg)^{-1} \nonumber \\
	&= \bigg( 
	\nabla^2 \mathcal{L}_t^{\mathrm{tr}} (\bar{\mathbf{v}}_t)
	\left[ \begin{matrix}
		\mathbf{D} & \mathbf{0}_d \\
		\mathbf{0}_d & 2\mathbf{D}^2
	\end{matrix} \right]
	\left[ \begin{matrix}
		\mathbf{I}_d & \mathbf{0}_d \\
		\diag \big( \nabla_{\bar{\mathbf{m}}_t} \mathcal{L}_t^{\mathrm{tr}} (\bar{\mathbf{v}}_t) \big) & \mathbf{I}_d
	\end{matrix} \right] \nonumber \\
	&~~~~~+ \left[ \begin{matrix}
		\mathbf{I}_d & \mathbf{0}_d \\
		\mathbf{0}_d & \big(\mathbf{I}_d + 2 \diag \big( \nabla_{\bar{\mathbf{d}}_t} \mathcal{L}_t^{\mathrm{tr}} (\bar{\mathbf{v}}_t) \big) \mathbf{D} \big)^2
	\end{matrix} \right]
	\left[ \begin{matrix}
		\mathbf{I}_d & \mathbf{0}_d \\
		\diag \big( \nabla_{\bar{\mathbf{m}}_t} \mathcal{L}_t^{\mathrm{tr}} (\bar{\mathbf{v}}_t) \big) & \mathbf{I}_d
	\end{matrix} \right]
	\bigg)^{-1} \nonumber \\
	&:= \big( \nabla^2 \mathcal{L}_t^{\mathrm{tr}} (\bar{\mathbf{v}}_t) \bar{\mathbf{P}}_t + \bar{\mathbf{Q}}_t \bar{\mathbf{R}}_t \big)^{-1},
\end{align}
where the third equality is from the definition of $\mathbf{H}_t (\bar{\mathbf{v}}_t)$. Likewise, we also have
\begin{align}
	\mathbf{G}_t (\hat{\mathbf{v}}_t) \mathbf{H}_t^{-1} (\hat{\mathbf{v}}_t)
	&= \bigg( 
	\nabla^2 \mathcal{L}_t^{\mathrm{tr}} (\hat{\mathbf{v}}_t)
	\left[ \begin{matrix}
		\mathbf{D} & \mathbf{0}_d \\
		\mathbf{0}_d & 2\mathbf{D}^2
	\end{matrix} \right]
	\left[ \begin{matrix}
		\mathbf{I}_d & \mathbf{0}_d \\
		\diag \big( \nabla_{\hat{\mathbf{m}}_t} \mathcal{L}_t^{\mathrm{tr}} (\hat{\mathbf{v}}_t) \big) & \mathbf{I}_d
	\end{matrix} \right] \nonumber \\
	&~~~~+ \left[ \begin{matrix}
		\mathbf{I}_d & \mathbf{0}_d \\
		\mathbf{0}_d & \big(\mathbf{I}_d + 2 \diag \big( \nabla_{\hat{\mathbf{d}}_t} \mathcal{L}_t^{\mathrm{tr}} (\hat{\mathbf{v}}_t) \big) \mathbf{D} \big)^2
	\end{matrix} \right]
	\left[ \begin{matrix}
		\mathbf{I}_d & \mathbf{0}_d \\
		\diag \big( \nabla_{\hat{\mathbf{m}}_t} \mathcal{L}_t^{\mathrm{tr}} (\hat{\mathbf{v}}_t) \big) & \mathbf{I}_d
	\end{matrix} \right]
	\bigg)^{-1} \nonumber \\
	&~~~:= \big( \nabla^2 \mathcal{L}_t^{\mathrm{tr}} (\hat{\mathbf{v}}_t) \hat{\mathbf{P}}_t + \hat{\mathbf{Q}}_t \hat{\mathbf{R}}_t \big)^{-1}.
\end{align}
Upon defining $\Delta := \big( \nabla \bar{\mathcal{A}}_t (\boldsymbol{\theta}) \big)^{-1} - \big( \mathbf{G}_t (\hat{\mathbf{v}}_t) \mathbf{H}_t^{-1} (\hat{\mathbf{v}}_t) \big)^{-1}$, and adding intermediate terms, we arrive at
\begin{align}
\label{eq:upper-bound-Delta-4terms}
	\| \Delta \|_2 
	&= \left\| 
	\big( \nabla \bar{\mathcal{A}}_t (\boldsymbol{\theta}) \big)^{-1} 
	- \nabla^2 \mathcal{L}_t^{\mathrm{tr}} (\bar{\mathbf{v}}_t) \hat{\mathbf{P}}_t - \bar{\mathbf{Q}}_t \hat{\mathbf{R}}_t 
	+ \nabla^2 \mathcal{L}_t^{\mathrm{tr}} (\bar{\mathbf{v}}_t) \hat{\mathbf{P}}_t + \bar{\mathbf{Q}}_t \hat{\mathbf{R}}_t 
	- \big( \mathbf{G}_t (\hat{\mathbf{v}}_t) \mathbf{H}_t^{-1} (\hat{\mathbf{v}}_t) \big)^{-1} \right\|_2 \nonumber \\
	&= \left\| \nabla^2 \mathcal{L}_t^{\mathrm{tr}} (\bar{\mathbf{v}}_t) \big( \bar{\mathbf{P}}_t - \hat{\mathbf{P}}_t \big) 
	+ \bar{\mathbf{Q}}_t \big( \bar{\mathbf{R}}_t  - \hat{\mathbf{R}}_t \big) 
	+ \big( \nabla^2 \mathcal{L}_t^{\mathrm{tr}} (\bar{\mathbf{v}}_t) - \nabla^2 \mathcal{L}_t^{\mathrm{tr}} (\hat{\mathbf{v}}_t) \big) \hat{\mathbf{P}}_t
	+ \big( \bar{\mathbf{Q}}_t - \hat{\mathbf{Q}}_t \big) \hat{\mathbf{R}}_t \right\|_2 
	\nonumber \\
	&\le \left\| \nabla^2 \mathcal{L}_t^{\mathrm{tr}} (\bar{\mathbf{v}}_t) \big( \bar{\mathbf{P}}_t - \hat{\mathbf{P}}_t \big) \right\|_2
	+ \left\| \bar{\mathbf{Q}}_t ( \bar{\mathbf{R}}_t  - \hat{\mathbf{R}}_t ) \right\|_2
	+ \left\| \big( \nabla^2 \mathcal{L}_t^{\mathrm{tr}} (\bar{\mathbf{v}}_t) - \nabla^2 \mathcal{L}_t^{\mathrm{tr}} (\hat{\mathbf{v}}_t) \big) \hat{\mathbf{P}}_t \right\|_2
	+ \left\| \big( \bar{\mathbf{Q}}_t - \hat{\mathbf{Q}}_t \big) \hat{\mathbf{R}}_t \right\|_2.
\end{align}
Next, we will bound the four summands in~\eqref{eq:upper-bound-Delta}. Using Assumption~\ref{as:task-nll}, it follows that  
\begin{align}
	\left\| \nabla^2 \mathcal{L}_t^{\mathrm{tr}} (\bar{\mathbf{v}}_t) \big( \bar{\mathbf{P}}_t - \hat{\mathbf{P}}_t \big) \right\|_2 
	&\le 
	\left\| \nabla^2 \mathcal{L}_t^{\mathrm{tr}} (\bar{\mathbf{v}}_t) \right\|_2
	\left\| \left[ 
	\begin{matrix}
		\mathbf{D} & \mathbf{0}_d \\
		\mathbf{0}_d & 2\mathbf{D}^2
	\end{matrix} 
	\right] \right\|_2
	\left\| \left[ 
	\begin{matrix}
		\mathbf{0}_d & \mathbf{0}_d \\
		\diag \big( \nabla_{\bar{\mathbf{m}}_t} \mathcal{L}_t^{\mathrm{tr}} (\bar{\mathbf{v}}_t) - \nabla_{\hat{\mathbf{m}}_t} \mathcal{L}_t^{\mathrm{tr}} (\hat{\mathbf{v}}_t) \big) & \mathbf{0}_d
	\end{matrix} 
	\right] \right\|_2 \nonumber \\
	&\le D_t \max \big\{ D_{\max}, 2 D^2_{\max} \big\} \| \nabla_{\bar{\mathbf{m}}_t} \mathcal{L}_t^{\mathrm{tr}} (\bar{\mathbf{v}}_t) - \nabla_{\hat{\mathbf{m}}_t} \mathcal{L}_t^{\mathrm{tr}} (\hat{\mathbf{v}}_t) \|_{\infty} \nonumber \\
	&\le D_t^2 \max \big\{ D_{\max}, 2 D^2_{\max} \big\} \| \bar{\mathbf{m}}_t - \hat{\mathbf{m}}_t \|_{\infty} \nonumber \\
	&\le D_t^2 \max \big\{ D_{\max}, 2 D^2_{\max} \big\} \epsilon_t,
\end{align}
and
\begin{align}
	\left\| \big( \nabla^2 \mathcal{L}_t^{\mathrm{tr}} (\bar{\mathbf{v}}_t) - \nabla^2 \mathcal{L}_t^{\mathrm{tr}} (\hat{\mathbf{v}}_t) \big) \hat{\mathbf{P}}_t \right\|_2
	&\le \left\| 
	\nabla^2 \mathcal{L}_t^{\mathrm{tr}} (\bar{\mathbf{v}}_t) - \nabla^2 \mathcal{L}_t^{\mathrm{tr}} (\hat{\mathbf{v}}_t) 
	\right\|_2
	\left\| \left[ 
	\begin{matrix}
		\mathbf{D} & \mathbf{0}_d \\
		\mathbf{0}_d & 2\mathbf{D}^2
	\end{matrix} 
	\right] \right\|_2
	\left\| \left[ 
	\begin{matrix}
		\mathbf{I}_d & \mathbf{0}_d \\
		\diag \big( \nabla_{\hat{\mathbf{m}}_t} \mathcal{L}_t^{\mathrm{tr}} (\hat{\mathbf{v}}_t) \big) & \mathbf{I}_d
	\end{matrix} 
	\right] \right \|_2 \nonumber \\
	&\le E_t \epsilon_t \max \big\{ D_{\max}, 2 D^2_{\max} \big\} \left\| \left[ \begin{matrix}
		\mathbf{I}_d & \mathbf{0}_d \\
		\diag \big( \nabla_{\hat{\mathbf{m}}_t} \mathcal{L}_t^{\mathrm{tr}} (\hat{\mathbf{v}}_t) \big) & \mathbf{I}_d
	\end{matrix} \right] \right\|_2 \nonumber \\
	&= E_t \epsilon_t \max \big\{ D_{\max}, 2 D^2_{\max} \big\} 
	\left( 1 + \left\| \left[ 
	\begin{matrix}
		\mathbf{0}_d & \mathbf{0}_d \\
		\diag \big( \nabla_{\hat{\mathbf{m}}_t} \mathcal{L}_t^{\mathrm{tr}} (\hat{\mathbf{v}}_t) \big) & \mathbf{0}_d
	\end{matrix} 
	\right] \right\|_2 \right) \nonumber \\
	&= E_t \max \big\{ D_{\max}, 2 D^2_{\max} \big\} (1 + \| \nabla_{\hat{\mathbf{m}}_t} \mathcal{L}_t^{\mathrm{tr}} (\hat{\mathbf{v}}_t) \|_{\infty}) \epsilon_t \nonumber \\
	&\le E_t \max \big\{ D_{\max}, 2 D^2_{\max} \big\} (1 + \rho_t) \epsilon_t.
\end{align}

Letting $\mathbf{v}^2$ denote the element-wise square of a general vector $\mathbf{v}$, we have for the second term that
\begin{align}
	\left\| \bar{\mathbf{Q}}_t ( \bar{\mathbf{R}}_t  - \hat{\mathbf{R}}_t ) \right\|_2
	&\le \left\| \left[ 
	\begin{matrix}
		\mathbf{I}_d & \mathbf{0}_d \\
		\mathbf{0}_d & \big(\mathbf{I}_d + 2 \diag \big( \nabla_{\bar{\mathbf{d}}_t} \mathcal{L}_t^{\mathrm{tr}} (\bar{\mathbf{v}}_t) \big) \mathbf{D} \big)^2
	\end{matrix} 
	\right] \right\|_2
	\left\| \left[ 
	\begin{matrix}
		\mathbf{0}_d & \mathbf{0}_d \\
		\diag \big( \nabla_{\bar{\mathbf{m}}_t} \mathcal{L}_t^{\mathrm{tr}} (\bar{\mathbf{v}}_t) - \nabla_{\hat{\mathbf{m}}_t} \mathcal{L}_t^{\mathrm{tr}} (\hat{\mathbf{v}}_t) \big) & \mathbf{0}_d
	\end{matrix} 
	\right] \right\|_2 \nonumber \\
	&= \max \big\{ 1, \big \| \big( \mathbf{1}_d + 2\mathbf{d} \cdot \nabla_{\bar{\mathbf{d}}_t} \mathcal{L}_t^{\mathrm{tr}} (\bar{\mathbf{v}}_t) \big)^2 \big \|_{\infty} \big\} \big \| \nabla_{\bar{\mathbf{m}}_t} \mathcal{L}_t^{\mathrm{tr}} (\bar{\mathbf{v}}_t) - \nabla_{\hat{\mathbf{m}}_t} \mathcal{L}_t^{\mathrm{tr}} (\hat{\mathbf{v}}_t) \big \|_{\infty} \nonumber \\
	&\le \max \big\{ 1, \big \| \mathbf{1}_d + 2\mathbf{d} \cdot \nabla_{\bar{\mathbf{m}}_t} \mathcal{L}_t^{\mathrm{tr}} (\bar{\mathbf{v}}_t) \big \|_{\infty}^2 \big\} D_t \| \bar{\mathbf{m}}_t - \hat{\mathbf{m}}_t \|_{\infty} \nonumber \\
	&\le D_t (1 + 2\max \big\{ D_{\max}, 2 D^2_{\max} \big\} \rho_t)^2 \epsilon_t,
\end{align}
where for the fourth inequality we employed Assumption~\ref{as:task-nll}, and the definition $\mathbf{1}_d := [1,\ldots,1]^\top \in \mathbb{R}^d$. 

For the last term, it holds that
\begin{align}
\label{eq:upper-bound-Delta-term4}
	&\left\| \big( \bar{\mathbf{Q}}_t - \hat{\mathbf{Q}}_t \big) \hat{\mathbf{R}}_t \right\|_2 \nonumber \\
	&\le \left \| \left[ \begin{matrix}
		\mathbf{0}_d & \mathbf{0}_d \\
		\mathbf{0}_d & \big(\mathbf{I}_d + 2 \diag \big( \nabla_{\bar{\mathbf{d}}_t} \mathcal{L}_t^{\mathrm{tr}} (\bar{\mathbf{v}}_t) \big) \mathbf{D} \big)^2 - \big(\mathbf{I}_d + 2 \diag \big( \nabla_{\hat{\mathbf{d}}_t} \mathcal{L}_t^{\mathrm{tr}} (\hat{\mathbf{v}}_t) \big) \mathbf{D} \big)^2
	\end{matrix} \right] \right\|_2
	\left \| \left[ 
	\begin{matrix}
		\mathbf{I}_d & \mathbf{0}_d \\
		\diag \big( \nabla_{\hat{\mathbf{m}}_t} \mathcal{L}_t^{\mathrm{tr}} (\hat{\mathbf{v}}_t) \big) & \mathbf{I}_d
	\end{matrix} 
	\right] \right\| \nonumber \\
	&= \left\| \big( \mathbf{1}_d + 2 \mathbf{d} \cdot \nabla_{\bar{\mathbf{d}}_t} \mathcal{L}_t^{\mathrm{tr}} (\bar{\mathbf{v}}_t) \big)^2 - \big( \mathbf{1}_d + 2 \mathbf{d} \cdot \nabla_{\hat{\mathbf{d}}_t} \mathcal{L}_t^{\mathrm{tr}} (\hat{\mathbf{v}}_t) \big)^2 \right\|_{\infty} \left( 1 + \left\| \nabla_{\hat{\mathbf{m}}_t} \mathcal{L}_t^{\mathrm{tr}} (\hat{\mathbf{v}}_t) \right\|_{\infty} \right) \nonumber \\
	&\le \left\| \big( \mathbf{1}_d + 2 \mathbf{d} \cdot \nabla_{\bar{\mathbf{d}}_t} \mathcal{L}_t^{\mathrm{tr}} (\bar{\mathbf{v}}_t) \big)^2 - \big( \mathbf{1}_d + 2 \mathbf{d} \cdot \nabla_{\hat{\mathbf{d}}_t} \mathcal{L}_t^{\mathrm{tr}} (\hat{\mathbf{v}}_t) \big)^2 \right\|_{\infty} (1 + \rho_t) \nonumber \\
	&= \left\| 2 \big( \mathbf{1}_d + \mathbf{d} \cdot (\nabla_{\bar{\mathbf{d}}_t} \mathcal{L}_t^{\mathrm{tr}} (\bar{\mathbf{v}}_t) + \nabla_{\hat{\mathbf{d}}_t} \mathcal{L}_t^{\mathrm{tr}} (\hat{\mathbf{v}}_t)) \big) \cdot 2 \big( \mathbf{d} \cdot (\nabla_{\bar{\mathbf{d}}_t} \mathcal{L}_t^{\mathrm{tr}} (\bar{\mathbf{v}}_t) - \nabla_{\hat{\mathbf{d}}_t} \mathcal{L}_t^{\mathrm{tr}} (\hat{\mathbf{v}}_t)) \big) \right\|_{\infty} (1 + \rho_t) \nonumber \\
	&\le 4 \big\| \mathbf{1}_d + \mathbf{d} \cdot (\nabla_{\bar{\mathbf{d}}_t} \mathcal{L}_t^{\mathrm{tr}} (\bar{\mathbf{v}}_t) + \nabla_{\hat{\mathbf{d}}_t} \mathcal{L}_t^{\mathrm{tr}} (\hat{\mathbf{v}}_t)) \big\|_{\infty} \big\| \mathbf{d} \cdot (\nabla_{\bar{\mathbf{d}}_t} \mathcal{L}_t^{\mathrm{tr}} (\bar{\mathbf{v}}_t) - \nabla_{\hat{\mathbf{d}}_t} \mathcal{L}_t^{\mathrm{tr}} (\hat{\mathbf{v}}_t)) \big\|_{\infty} (1 + \rho_t) \nonumber \\
	&\le 4 (1 + \max \big\{ D_{\max}, 2 D^2_{\max} \big\} \| \nabla_{\bar{\mathbf{d}}_t} \mathcal{L}_t^{\mathrm{tr}} (\bar{\mathbf{v}}_t) + \nabla_{\hat{\mathbf{d}}_t} \mathcal{L}_t^{\mathrm{tr}} (\hat{\mathbf{v}}_t) \|_{\infty}) \max \big\{ D_{\max}, 2 D^2_{\max} \big\} \| \nabla_{\bar{\mathbf{d}}_t} \mathcal{L}_t^{\mathrm{tr}} (\bar{\mathbf{v}}_t) - \nabla_{\hat{\mathbf{d}}_t} \mathcal{L}_t^{\mathrm{tr}} (\hat{\mathbf{v}}_t) \|_{\infty} (1 + \rho_t) \nonumber \\
	&\overset{(a)}{\le} 4 (1 + 2\max \big\{ D_{\max}, 2 D^2_{\max} \big\}\rho_t) \max \big\{ D_{\max}, 2 D^2_{\max} \big\} D_t \| \mathbf{d}_t - \hat{\mathbf{d}}_t \|_{\infty} (1 + \rho_t) \nonumber \\
	&\le 4 D_t \max \big\{ D_{\max}, 2 D^2_{\max} \big\} (1 + 2\max \big\{ D_{\max}, 2 D^2_{\max} \big\}\rho_t)(1 + \rho_t) \epsilon_t,
\end{align}
where $(a)$ utilizes Assumption~\ref{as:task-nll}. 

Combining~\eqref{eq:upper-bound-Delta-4terms}-\eqref{eq:upper-bound-Delta-term4}, we arrive at
\begin{align}
\label{eq:upper-bound-Delta}
	\| \Delta \|_2 \le 
	&\big\{ D_t^2 \max \big\{ D_{\max}, 2 D^2_{\max} \big\} + E_t \max \big\{ D_{\max}, 2 D^2_{\max} \big\} (1 + \rho_t) + D_t (1 + 2\max \big\{ D_{\max}, 2 D^2_{\max} \big\} \rho_t)^2 \nonumber \\
	&+ 4 D_t \max \big\{ D_{\max}, 2 D^2_{\max} \big\} (1 + 2\max \big\{ D_{\max}, 2 D^2_{\max} \big\}\rho_t)(1 + \rho_t) \big\} \epsilon_t \nonumber \\
	&:= F_t^{\Delta} \epsilon_t.
\end{align}

Further, we can use Assumption~\ref{as:invertible} to establish one of the desired upper bounds
\begin{align}
\label{eq:target-upper-bound-1}
	\| \nabla \bar{\mathcal{A}}_t (\boldsymbol{\theta}) \|_2 
	&\le \left\| \left[ 
	\begin{matrix}
		\mathbf{D}^{-1} & \mathbf{0}_d \\
		-\diag \big( \nabla_{\bar{\mathbf{m}}_t} \mathcal{L}_t^{\mathrm{tr}} (\bar{\mathbf{v}}_t) \big) & \frac{1}{2}\mathbf{D}^{-2}
	\end{matrix} 
	\right] \right\|_2
	\left\| \mathbf{H}_t^{-1} (\hat{\mathbf{v}}_t) \right\|_2 \nonumber \\
	&\le \left\| \left[ 
	\begin{matrix}
		\mathbf{I}_d & \mathbf{0}_d \\
		-\diag \big( \nabla_{\bar{\mathbf{m}}_t} \mathcal{L}_t^{\mathrm{tr}} (\bar{\mathbf{v}}_t) \big) & \mathbf{I}_d
	\end{matrix} 
	\right] \right\|_2
	\left\| \left[ 
	\begin{matrix}
		\mathbf{D}^{-1} & \mathbf{0}_d \\
		\mathbf{0}_d & \frac{1}{2}\mathbf{D}^{-2}
	\end{matrix} 
	\right] \right\|_2
	\sigma_t^{-1} \nonumber \\
	&= \big( 1 + \| \nabla_{\bar{\mathbf{m}}_t} \mathcal{L}_t^{\mathrm{tr}} (\bar{\mathbf{v}}_t) \|_{\infty} \big) \max\{ \| \mathbf{d}^{-1} \|_{\infty}, \| \frac{1}{2} \mathbf{d}^{-2} \|_{\infty} \} \sigma_t^{-1} \nonumber \\
	&\le (1 + \rho_t) \max \big\{ D^{-1}_{\min}, \frac{1}{2} D^{-2}_{\min} \big\} \sigma_t^{-1},
\end{align}
and likewise
\begin{equation}
\label{eq:GH-upper}
	\| \mathbf{G}_t (\hat{\mathbf{v}}_t) \mathbf{H}_t^{-1} (\hat{\mathbf{v}}_t) \|_2 \le (1 + \rho_t) \max \big\{ D^{-1}_{\min}, \frac{1}{2} D^{-2}_{\min} \big\} \sigma_t^{-1}.
\end{equation}

Through~\eqref{eq:upper-bound-Delta} and~\eqref{eq:GH-upper}, we can also  establish the other upper bound as
\begin{align}
\label{eq:target-upper-bound-2}
	\| \nabla \hat{\mathcal{A}}_t (\boldsymbol{\theta}) - \nabla \bar{\mathcal{A}}_t (\boldsymbol{\theta}) \|_2 
	&\le \| \nabla \hat{\mathcal{A}}_t (\boldsymbol{\theta}) - \mathbf{G}_t (\hat{\mathbf{v}}_t) \mathbf{H}_t^{-1} (\hat{\mathbf{v}}_t) \|_2 + \| \mathbf{G}_t (\hat{\mathbf{v}}_t) \mathbf{H}_t^{-1} (\hat{\mathbf{v}}_t) - \nabla \bar{\mathcal{A}}_t (\boldsymbol{\theta}) \|_2 \nonumber \\
	&= \delta_t + \| \nabla \bar{\mathcal{A}}_t (\boldsymbol{\theta}) \Delta \mathbf{G}_t (\hat{\mathbf{v}}_t) \mathbf{H}_t^{-1} (\hat{\mathbf{v}}_t) \|_2 \nonumber \\
	&\le \delta_t + \| \nabla \bar{\mathcal{A}}_t (\boldsymbol{\theta}) \|_2 \| \Delta \|_2 \| \mathbf{G}_t (\hat{\mathbf{v}}_t) \mathbf{H}_t^{-1} (\hat{\mathbf{v}}_t) \|_2 \nonumber \\
	&\le \delta_t + (1 + \rho_t)^2 \min \big\{ D_{\min}, 2D^2_{\min} \big\}^{-2} \sigma_t^{-2} F_t^{\Delta} \epsilon_t. 
\end{align}

Finally, relating~\eqref{eq:target-upper-bound-1} and~\eqref{eq:target-upper-bound-2} to~\eqref{eq:theor1-target} completes the proof of the theorem. 
\end{proof}

\subsection*{A.3 Proof of Theorem~\ref{theor:implicit}}
\begin{theorem}[implicit gradient error bound, restated]
Consider the Bayesian meta-learning problem in~\eqref{eq:Bayesian-subopt}. Let $\epsilon_t := \| \hat{\mathbf{v}}_t - \bar{\mathbf{v}}_t \|_2$ be the task-level optimization error, and $\delta_t' := \| \hat{\mathbf{u}}_t - \mathbf{H}_t^{-1} (\hat{\mathbf{v}}_t) \nabla_1 \mathcal{L}_t^{\mathrm{val}} (\hat{\mathbf{v}}_t, \boldsymbol{\theta}) \|$ the CG error. Upon defining $\rho_t := \max \big\{ \| \nabla_{\bar{\mathbf{v}}_t} \mathcal{L}_t^{\mathrm{tr}} (\bar{\mathbf{v}}_t) \|_{\infty}, \| \nabla_{\hat{\mathbf{v}}_t} \mathcal{L}_t^{\mathrm{tr}} (\hat{\mathbf{v}}_t) \|_{\infty} \big\}$, and with Assumptions~\ref{as:local-min}-\ref{as:bounded-var} in effect, it holds for $t \in \{ 1,\ldots, T \}$ that
\begin{equation}
	\big\| \hat{\mathbf{g}}_t -\nabla_{\boldsymbol{\theta}} \mathcal{L}_t^{\mathrm{val}} \big( \bar{\mathbf{v}}_t (\boldsymbol{\theta}), \boldsymbol{\theta} \big) \big\|_2 \le F_t' \epsilon_t + G_t' \delta_t',
\end{equation}
where $F_t'$ and $G_t'$ are scalars not dependent on $\rho_t$. 
\end{theorem}

\begin{proof}
From~\eqref{eq:Bayesian-meta-grad} and~\eqref{eq:implicit-grad}, we deduce that
\begin{align}
\label{eq:implicit-bound-1}
	&\big\| \hat{\mathbf{g}}_t -\nabla_{\boldsymbol{\theta}} \mathcal{L}_t^{\mathrm{val}} \big( \bar{\mathbf{v}}_t (\boldsymbol{\theta}), \boldsymbol{\theta} \big) \big\|_2 \nonumber \\
	&\le \big\| \mathbf{G}_t (\hat{\mathbf{v}}_t) \hat{\mathbf{u}}_t
	- \nabla \bar{\mathcal{A}}_t (\boldsymbol{\theta}) \nabla_1 \mathcal{L}_t^{\mathrm{val}}(\bar{\mathbf{v}}_t, \boldsymbol{\theta}) \big\|_2
	+ \big\| \nabla_2 \mathcal{L}_t^{\mathrm{val}}(\hat{\mathbf{v}}_t, \boldsymbol{\theta}) - \nabla_2 \mathcal{L}_t^{\mathrm{val}}(\bar{\mathbf{v}}_t, \boldsymbol{\theta}) 
	\big\|_2 \nonumber \\
	&\overset{(a)}{\le} \big\| \mathbf{G}_t (\hat{\mathbf{v}}_t) \hat{\mathbf{u}}_t
	- \nabla \bar{\mathcal{A}}_t (\boldsymbol{\theta}) \nabla_1 \mathcal{L}_t^{\mathrm{val}}(\bar{\mathbf{v}}_t, \boldsymbol{\theta}) \big\|_2
	+ C_t \epsilon_t \nonumber \\
	&= \big\| \mathbf{G}_t (\hat{\mathbf{v}}_t) \hat{\mathbf{u}}_t
	- \mathbf{G}_t (\bar{\mathbf{v}}_t) \mathbf{H}_t^{-1} (\bar{\mathbf{v}}_t) \nabla_1 \mathcal{L}_t^{\mathrm{val}}(\bar{\mathbf{v}}_t, \boldsymbol{\theta}) \big\|_2
	+ C_t \epsilon_t \nonumber \\
	&\le \big\| \mathbf{G}_t (\hat{\mathbf{v}}_t) \big( \hat{\mathbf{u}}_t - \mathbf{H}_t^{-1} (\hat{\mathbf{v}}_t) \nabla_1 \mathcal{L}_t^{\mathrm{val}} (\hat{\mathbf{v}}_t, \boldsymbol{\theta}) \big) \big\|_2  \nonumber \\
	&~~~~~+ \big\| \mathbf{G}_t (\hat{\mathbf{v}}_t) \mathbf{H}_t^{-1} (\hat{\mathbf{v}}_t) \nabla_1 \mathcal{L}_t^{\mathrm{val}} (\hat{\mathbf{v}}_t, \boldsymbol{\theta}) 
	- \nabla \bar{\mathcal{A}}_t (\boldsymbol{\theta}) \nabla_1 \mathcal{L}_t^{\mathrm{val}}(\bar{\mathbf{v}}_t, \boldsymbol{\theta}) \big\|_2
	+ C_t \epsilon_t \nonumber \\
	&\overset{(b)}{\le} (1 + \rho_t) \max \big\{ D^{-1}_{\min}, \frac{1}{2} D^{-2}_{\min} \big\} \delta_t' 
	+ \big\| \mathbf{G}_t (\hat{\mathbf{v}}_t) \mathbf{H}_t^{-1} (\hat{\mathbf{v}}_t) \nabla_1 \mathcal{L}_t^{\mathrm{val}} (\hat{\mathbf{v}}_t, \boldsymbol{\theta}) 
	- \nabla \bar{\mathcal{A}}_t (\boldsymbol{\theta}) \nabla_1 \mathcal{L}_t^{\mathrm{val}}(\bar{\mathbf{v}}_t, \boldsymbol{\theta}) \big\|_2
	+ C_t \epsilon_t,
\end{align}
where $(a)$ comes from Assumption~\ref{as:meta-loss}, and $(b)$ uses that 
\begin{align}
	\big\| \mathbf{G}_t (\hat{\mathbf{v}}_t) \big\|_2 
	&= \left\| \left[ 
	\begin{matrix}
		\mathbf{D}^{-1} & \mathbf{0}_d \\
		-\diag \big( \nabla_{\hat{\mathbf{m}}_t} \mathcal{L}_t^{\mathrm{tr}} (\hat{\mathbf{v}}_t) \big) \mathbf{D}^{-1}	& \frac{1}{2}\mathbf{D}^{-2}
	\end{matrix} 
	\right] \right\|_2 \nonumber \\
	&\le \left\| \left[ 
	\begin{matrix}
		\mathbf{I}_d & \mathbf{0}_d \\
		-\diag \big( \nabla_{\hat{\mathbf{m}}_t} \mathcal{L}_t^{\mathrm{tr}} (\hat{\mathbf{v}}_t) \big) & \mathbf{I}_d
	\end{matrix} 
	\right] \right\|_2
	\left\| \left[ 
	\begin{matrix}
		\mathbf{D}^{-1} & \mathbf{0}_d \\
		\mathbf{0}_d & \frac{1}{2}\mathbf{D}^{-2}
	\end{matrix} 
	\right] \right\|_2 \nonumber \\
	&\le (1 + \rho_t) \max \big\{ D^{-1}_{\min}, \frac{1}{2} D^{-2}_{\min} \big\}.
\end{align}

To bound $\| \mathbf{G}_t (\hat{\mathbf{v}}_t) \mathbf{H}_t^{-1} (\hat{\mathbf{v}}_t) \nabla_1 \mathcal{L}_t^{\mathrm{val}} (\hat{\mathbf{v}}_t, \boldsymbol{\theta}) - \nabla \bar{\mathcal{A}}_t (\boldsymbol{\theta}) \nabla_1 \mathcal{L}_t^{\mathrm{val}}(\bar{\mathbf{v}}_t, \boldsymbol{\theta}) \|_2$, we again add intermediate terms to arrive at
\begin{align}
\label{eq:implicit-bound-2}
	&\big\| \mathbf{G}_t (\hat{\mathbf{v}}_t) \mathbf{H}_t^{-1} (\hat{\mathbf{v}}_t) \nabla_1 \mathcal{L}_t^{\mathrm{val}} (\hat{\mathbf{v}}_t, \boldsymbol{\theta}) 
	- \nabla \bar{\mathcal{A}}_t (\boldsymbol{\theta}) \nabla_1 \mathcal{L}_t^{\mathrm{val}}(\bar{\mathbf{v}}_t, \boldsymbol{\theta}) \big\|_2 \nonumber \\
	&\le \big\| \mathbf{G}_t (\hat{\mathbf{v}}_t) \mathbf{H}_t^{-1} (\hat{\mathbf{v}}_t) \nabla_1 \mathcal{L}_t^{\mathrm{val}} (\hat{\mathbf{v}}_t, \boldsymbol{\theta}) 
	- \mathbf{G}_t (\hat{\mathbf{v}}_t) \mathbf{H}_t^{-1} (\hat{\mathbf{v}}_t) \nabla_1 \mathcal{L}_t^{\mathrm{val}} (\bar{\mathbf{v}}_t, \boldsymbol{\theta}) \big\|_2 \nonumber \\
	&~~~~~+ \big\| \mathbf{G}_t (\hat{\mathbf{v}}_t) \mathbf{H}_t^{-1} (\hat{\mathbf{v}}_t) \nabla_1 \mathcal{L}_t^{\mathrm{val}}(\bar{\mathbf{v}}_t, \boldsymbol{\theta}) 
	- \nabla \bar{\mathcal{A}}_t (\boldsymbol{\theta}) \nabla_1 \mathcal{L}_t^{\mathrm{val}}(\bar{\mathbf{v}}_t, \boldsymbol{\theta}) \big\|_2 \nonumber \\
	&\le \big\| \mathbf{G}_t (\hat{\mathbf{v}}_t) \mathbf{H}_t^{-1} (\hat{\mathbf{v}}_t) \big\|_2 
	\big\| \nabla_1 \mathcal{L}_t^{\mathrm{val}} (\hat{\mathbf{v}}_t, \boldsymbol{\theta}) - \nabla_1 \mathcal{L}_t^{\mathrm{val}} (\bar{\mathbf{v}}_t, \boldsymbol{\theta}) \big\|_2 
	+ \big\| \mathbf{G}_t (\hat{\mathbf{v}}_t) \mathbf{H}_t^{-1} (\hat{\mathbf{v}}_t) - \nabla \bar{\mathcal{A}}_t (\boldsymbol{\theta}) \big\|_2 \big\| \nabla_1 \mathcal{L}_t^{\mathrm{val}}(\bar{\mathbf{v}}_t, \boldsymbol{\theta}) \big\|_2 \nonumber \\
	&\le (1 + \rho_t) \max \big\{ D^{-1}_{\min}, \frac{1}{2} D^{-2}_{\min} \big\} \sigma_t^{-1} B_t \epsilon_t + (1 + \rho_t)^2 \max \big\{ D^{-2}_{\min}, \frac{1}{4} D^{-4}_{\min} \big\} \sigma_t^{-2} F_t^{\Delta} \epsilon_t A_t \nonumber \\
	&= \big( B_t (1 + \rho_t) \max \big\{ D^{-1}_{\min}, \frac{1}{2} D^{-2}_{\min} \big\} \sigma_t^{-1} + A_t (1 + \rho_t)^2 \max \big\{ D^{-2}_{\min}, \frac{1}{4} D^{-4}_{\min} \big\} \sigma_t^{-2} F_t^{\Delta} \big) \epsilon_t
\end{align}
where the third inequality follows from~\eqref{eq:GH-upper},~\eqref{eq:target-upper-bound-2} and Assumption~\ref{as:meta-loss}. 

Plugging~\eqref{eq:implicit-bound-2} into~\eqref{eq:implicit-bound-1} completes the proof of the theorem.

\end{proof}

\subsection*{A.4 Detailed setups for numerical tests}
\paragraph{Synthetic dataset}
Across all tests, the dimension $d=32$, and the standard deviation of AWGN is $\sigma = 0.01$. Matrix $\mathbf{X}_t^{\mathrm{tr}}$ is randomly generated with condition number $\kappa = 20$, and the linear weights are randomly sampled from the oracle distribution $p(\boldsymbol{\theta}_t; \boldsymbol{\theta}^*) = \mathcal{N} (\mathbf{0}_d, \mathbf{I}_d)$. The size of the training and validation sets are fixed as $| \mathcal{D}_t^{\mathrm{tr}} | = 32$ and $| \mathcal{D}_t^{\mathrm{val}} | = 64$ for $t \in \{ 1, \ldots, T \}$. The task-level optimization function $\hat{\mathcal{A}}_t$ is chosen to be the $K$-step GD with learning rate $\alpha = 0.01$. To run $\hat{\mathcal{A}}_t$ and compute the meta-loss in~\eqref{eq:Bayesian-subopt}, the number of Monte Carlo (MC) samples is set to $64$. 

\paragraph{\textit{Mini}ImageNet}
The numerical tests on \textit{mini}ImageNet follow the few-learning protocol described in~\cite{MatchingNets, MAML}. For meta-level optimization, the total number of iterations is $40,000$ with batch size $| \mathcal{B}^r | = 2$ and meta-learning rate $\beta = 0.001$. The meta-level prior of ABML is set to $\text{Gamma}(\mathbf{1}_d, 0.01 * \mathbf{1}_d)$ according to~\cite{ABML}. For task-level optimization, the learning rate is $\alpha = 0.01$. In addition, the number of MC runs is taken to be $5$ for meta-training, and $10$ for evaluation. 

Furthermore, to ensure that the entries $[\mathbf{d}]_i$ and $[\mathbf{d}_t]_i$ of the variances are greater than $0$, we instead optimize $\log [\mathbf{d}]_i$ and $\log [\mathbf{d}_t]_i$. This is possible because for a general $d$, it holds that  $\nabla_{\log d} f(d) = \nabla_{\log d} d \nabla_{d} f(d) = d \nabla f(d)$. 
\end{document}